\RequirePackage[l2tabu, orthodox]{nag}
\documentclass[english]{article}     

\usepackage[utf8]{inputenc} %
\usepackage[T1]{fontenc}    %
\usepackage{hyperref}       %
\usepackage{url}            %
\usepackage{booktabs}       %
\usepackage{amsfonts}       %
\usepackage{nicefrac}       %
\usepackage{microtype}      %

\usepackage{mathptmx}    %
\usepackage[english]{babel}

\usepackage{amsmath}
\usepackage{amssymb}
\usepackage{amsthm}
\usepackage{graphicx}
\usepackage{subcaption}
\usepackage{breakurl}
\usepackage{todonotes}
\usepackage{verbatim}
\usepackage{mathtools}
\mathtoolsset{showonlyrefs=true}
\usepackage[inline]{enumitem}
\usepackage{algorithm}
\usepackage{algpseudocode}
\usepackage{float}
\usepackage{siunitx}
\usepackage{icomma}
\usepackage{thmtools}
\usepackage{thm-restate}

\newtheorem{theorem}{Theorem}

\newcommand{\EE}{\mathbb{E}}
\newcommand{\R}{\mathbb{R}}

\newcommand{\bz}{\mathbf{z}}

\DeclareMathOperator{\argmin}{arg\min}

\newcommand{\cZ}{\mathcal{Z}}

\begin{document}

\title{Carath\'eodory Sampling for Stochastic Gradient Descent%
}

\author{Francesco Cosentino\thanks{\scriptsize{Mathematical Institute, University of Oxford  \& The Alan Turing Institute, \url{name.surname@maths.ox.ac.uk}}}
\and Harald Oberhauser\thanks{\scriptsize{Mathematical Institute, University of Oxford  \& The Alan Turing Institute, \url{name.surname@maths.ox.ac.uk}}}
\and Alessandro Abate\thanks{\scriptsize{Dept. of Computer Science,  University of Oxford  \& The Alan Turing Institute, \url{name.surname@cs.ox.ac.uk}}}
}

\date{}

\maketitle

\begin{abstract}
  Many problems require to optimize empirical risk functions over large data sets.
  Gradient descent methods that calculate the full gradient in every descent step do not scale to such datasets.
    Various flavours of Stochastic Gradient Descent (SGD) replace the expensive summation that computes the full gradient by approximating it with a small sum over a randomly selected subsample of the data set that in turn suffers from a high variance.
    We present a different approach that is inspired by classical results of Tchakaloff and Carath\'eodory about measure reduction.
    These results allow to replace an empirical measure with another, carefully constructed probability measure that has a much smaller support, but can preserve certain statistics such as the expected gradient.
    To turn this into scalable algorithms we firstly, adaptively select the descent steps where the measure reduction is carried out; secondly, we combine this with Block Coordinate Descent so that measure reduction can be done very cheaply.
    This makes the resulting methods scalable to high-dimensional spaces. 
    Finally, we provide an experimental validation and comparison.

\end{abstract}
\section{Introduction}\label{sec:intro}
A common taks is the optimization problem
  \begin{align}\label{eq:loss}
\argmin_{\theta } \EE[L_\theta(X,Y)],
  \end{align}
where $L_\theta$ is a generic loss function and the expectation is taken with respect to the joint distribution of $(X,Y)$; e.g.~$L_\theta(x,y)=(x^\top\theta - y)^2 + \lambda|\theta|_1$ for LASSO \cite{Tibshirani1996,Murphy2012}. %
In practice, the distribution of $(X,Y)$ is not known and one approximates it with the empirical measure $\mu=\frac{1}{N}\sum_{i=1}^N \delta_{(x_i,y_i)}$ built from $N$ samples $(x_i,y_i)$ of the pair $(X,Y)$.
That is, one minimizes the so-called empirical risk
\begin{align}\label{eq:emp loss}
\theta^\star:= \argmin_{\theta } \EE[L_{\theta}(Z)] = \frac{1}{N}\sum_{i=1}^N L_\theta(z_i), 
\end{align}
where $Z$ denotes the discrete random variable that takes the values $z_i=(x_i,y_i)$ with equal probability.
If $L_{\theta}$ is smooth with bounded derivatives and convex, standard gradient descent (GD),
\begin{align}\label{eq:GD1}
  \theta_{j+1}-\theta_{j}%
  =-\frac{\gamma}{N}\sum_{i=1}^N\nabla_{\theta}L_{\theta}(z_i)\Big|_{\theta=\theta_{i}}
\end{align}
converges if the learning rate $\gamma$ is appropriately chosen, $\lim_{i \rightarrow \infty} \theta_i =\theta^\star$~\cite{Nesterov2018}.
However, for large-scale problems when the number of samples $N$ is huge, the evaluation of the gradient in
Equation~\eqref{eq:GD1} in every iteration step is prohibitive.

\subsection{Related Literature}
Popular approaches to reduce the cost in each iteration step are the so-called stochastic gradient descent (SGD) algorithms.
The gradient is approximated by selecting at each iteration step $j$ a subset of the $N$ points at random.
Solving the minimization problem~\eqref{eq:emp loss} via (S)GD has a long history and is a research topic that is still rapidly evolving; we refer to \cite{Robbins1951,Bottou2004,bottou2010large,ruder2016overview,Bottou2018} for a general overview.
Our work is inspired by variants of SGD that produce better estimators for the expected gradient than the naive estimator given by subsampling a batch of data points uniformly at random.
This reduces the variance in each step and often guarantees a better convergence rate.
Popular examples of such SGD variants include stochastic average gradients (SAG) \cite{Roux2012}, Iterate Averaging \cite{Polyak1992}, Incremental Aggregated Gradient \cite{Blatt2007}. %

Our work relies on replacing the empirical measure $\mu=\frac{1}{N}\sum_{i=1}^N\delta_{z_i}$ by a measure $\hat \mu$ with much smaller support, 
which however has the property that $\EE_{Z \sim \mu}[\nabla_\theta L_\theta (Z)] = \EE_{Z \sim \hat \mu}[\nabla_\theta L_\theta (Z)]$ at carefully selected iteration steps.
The construction of this reduced measure $\hat \mu$ that matches certain expectations of $\mu$ is known as the {\it recombination problem}.  
Algorithms to construct this reduced measure by solving a constrained linear system have been known for a long time~\cite{Davis1967}, and recently more efficient algorithms have been developed~\cite{Litterer2012,maria2016a,Maalouf2019,Cosentino2020}.
In particular, we rely on \cite{Cosentino2020} that shows strong performance when the number of samples $N$ is very large. 

In the second part of this paper we combine our proposed Caratheodory subsampling with Block Coordinate Descent (BCD) to make Caratheodory sampling effective when $\theta$ is high-dimensional.
Generally, under the block separability assumptions on the regularization term and usual conditions on the principal part of the loss function, e.g. convexity or the Polyak-Lojasiewicz condition, the convergence is proved and the rates of convergence have been found, e.g.~\cite{Nutini2017,Csiba2017}.
The papers \cite{Tseng2007,Tseng2008} study how the smoothness assumptions can be relaxed.
Applications of BCD techniques have been studied in sparse settings~\cite{Nutini2017}, large scale Gaussian process regression~\cite{Bo2012}, L1 Regularized Least Squares~\cite{Santis2016}, Group LASSO~\cite{Meier2008,Qin2013}, Training Support Vector Machines~\cite{Platt1998},
matrix and tensor factorization~\cite{Xu2013,Yu2012} and other works~\cite{Ene2015,Sardy2000,Thoppe2014}.

\subsection{Contribution}
Instead of approximating the sum $\EE[L_\theta(Z)] = \frac{1}{N} \sum_{i=1}^N L_\theta(z)_i$ by subsampling, we construct at certain steps $j$ in the GD iteration of $\theta_j$ a new probability measure $\hat \mu_j$ supported on a \emph{very small subset} of the original $N$ atoms.
This measure $\hat \mu$, however, matches certain statistical functions of the empirical measure $\mu$; in particular, we can choose it such that  $\EE_{Z \sim \mu}[L_{\theta_j}(Z)] = \EE_{Z \sim \hat \mu}[L_{\theta_j}(Z)]$.
The construction of $\hat \mu$ is also known as the \emph{recombination problem} and we use the recent algorithm \cite{Cosentino2020} which scales well in the regime where the number of samples $N$ is large. 
Although it can be relatively costly to carry out the recombination at a given step, in return the gradient is perfectly matched at this step and the expectation can be computed as a sum over $n$ weighted points rather than $N\gg n$ uniformly weighted points. 
We balance this tradeoff by combining two techniques: 
\begin{enumerate*}[label=(\roman*)]
\item \label{itm:control} By using an approximation to the Hessian to derive a control statistic that tells us when to carry out the recombination step.
  In practice, this allows to do only a few recombination computations over the whole descent trajectory.
\item \label{itm:bc} By using Block coordinate descent (BCD) to carry out the reduction only for a subset of the %
coordinates of the gradient. This makes a recombination cheap even if $\theta$ is high-dimensional.
\end{enumerate*}

\subsection{Outline}
Section~\ref{sec:background} introduces the theoretical background on recombination.
Section \ref{sec:GD} provides the main theoretical results and a first application %
to logistic regression;
Section \ref{sec:BCD} then {recalls BCD techniques}, and shows how this allows to efficiently carry out Caratheodory subsampling for high-dimensional $\theta$.
Further, it benchmarks the resulting Carath\'eodory BCD (CaBCD) against classic SGD algorithms SAG and ADAM.
Section \ref{sec:BCD} also compares the rules regarding the selection of the coordinates' blocks introduced in~\cite{Nutini2017} with the same
rules when the Carath\'eodory Sampling is applied.
A Python implementation for all of our experiments can be found at \url{https://github.com/FraCose/Caratheodory_GD_Acceleration}.

\section{The Recombination Problem}\label{sec:background}
We now recall a classic result which shows that for any discrete random variable that can take $N$ different values, there exists another discrete random variable that only takes values in a subset of $n+1$ of the original $N$ points that has the same statistics as defined by $n$ functions $f_1,\ldots,f_n$. %

\begin{theorem}[Carath\'eodory~\cite{DeLoera2013} ]\label{th:cath}
Given a set of $N > n+1$ points in $\R^n$ and a point $z$ that lies in the convex hull of these $N$ points,
$z$ can be expressed as a convex combination of maximum $n+1$ points.
\end{theorem}
As is well-known, this implies Tchakaloff's Theorem~\ref{th:tchakalof}~\cite{Tchak} for the special case of discrete measures:
given $n$ functions $f_1,\ldots,f_n:\cZ \rightarrow \R^n$ define $F: \cZ \rightarrow \R^n$ as $F(z):=(f_1(z),\ldots,f_n(z))$.
Now given a discrete probability measure $\mu$ on $\cZ$ that is supported on $N$ atoms $z_1,\ldots,z_N \in\cZ$, it follows that $\EE_{Z\sim \mu }[F(Z)]= \sum_{i=1}^N F(z_i)\mu(z_i)$.
Since this finite sum
  defines a point within the convex hull of the set of $N$ points $\bz:=\{F(z_i)\}_{i=1}^N$, it follows by Carath\'eodory's Theorem that this point can be equivalently expressed as a convex combination of a subset $\hat \bz$ of $\bz$ comprising at most $n+1$ points.
  As first shown by Tchakaloff, this shows that Theorem~\ref{th:cath} implies the following recombination result. 
\begin{theorem}[Tchakaloff~\cite{Tchak}]\label{th:tchakalof}
  Let $Z$ be a discrete random variable that can take $N$ values $\{z_1,\ldots,z_N\}$.
  For any set $\{f_1,\ldots,f_n\}$ of $n$ real-valued functions there exists a random variable $\hat Z$ such that
\begin{align}\label{eq:F}
 \EE[f_i(Z)] = \EE[f_i(\hat Z)]\quad \text{ for every } i=1, \ldots, n.
 \end{align}
 and $\hat Z$ only takes values in a subset of $\{z_1,\ldots,z_N\}$ of cardinality at most $n+1$.
 We refer to $\hat Z$ as a reduction or recombination of $Z$.
\end{theorem}
Tchakaloff~\cite{Tchak} showed a more general version for continuous random variables, but in this work the above result for the discrete setting is sufficient.
In our context of the optimization problem~\eqref{eq:emp loss}, we will apply it with $Z=(X,Y)$ denoting a pair consisting of observations $X$ and labels $Y$.

The above derivation already implies an algorithm to calculate $\hat Z$, by finding the subset $\hat \bz$ , that solves $N-n-1$ times a constrained linear system, see~\cite{Davis1967} for details.
More recently, algorithms have been devised that exploit a divide and conquer strategy which reduce the complexity of the needed calculations drastically, but they all require ${O}(Nn+\log(N/n)n^4)$~\cite{Litterer2012,Maalouf2019} respectively ${O}(Nn+\log(N/n)n^3)$~\cite{maria2016a}.
Throughout we use~\cite{Cosentino2020} to construct $\hat \mu$ by a geometric greedy sampling which is advantageous when $N\gg n$.
In particular, what makes the algorithm \cite{Cosentino2020} suitable in contrast to other recombination algorithms \cite{Litterer2012,Maalouf2019,maria2016a} is that although it has a similar worst case complexity, it has a much better average case complexity.  
However, we emphasize that the ideas below are independent of the choice of the concrete recombination algorithm and any improvement on recombination algorithms will result in an improvement of Caratheodory's subsampling for SGD.
\section{Carath\'eodory Gradient Descent (CaGD)}\label{sec:GD}
Given a dataset $\{(x_i,y_i):i=1,\ldots,N\}$ consisting of $N$ observations $x_i$ with labels $y_i$, we denote by $Z$ the discrete random variable that takes the value $z_i=(x_i,y_i)$ with probability $\frac{1}{N}$.
That is, the empirical risk~\eqref{eq:emp loss} at $\theta$ equals $\EE[L_\theta(Z)]$.
Further, denote with $G(\theta,z):=\nabla_{\theta}L_\theta(z) \in \R^n$ the gradient at $\theta$ and with $H(\theta,z):=\nabla^2_{\theta}L_\theta(z) $ the Hessian.
With this notation, the usual GD iteration reads as
\begin{align}
 \theta_{j+1} := \theta_j - \gamma \EE[G(\theta,Z)],
\end{align}
and converges, under assumptions which we recall below, to the minimum $\theta^\star$ as given in~\eqref{eq:emp loss}, see \cite{Nesterov2018}.
However, in every descent step $j$ the evaluation of the sum \[\EE[G(\theta,Z)]=\frac{1}{N}\sum_{i=1}^N G(\theta,(x_i, y_i))\] can be costly. 
Below we use Theorem~\ref{th:tchakalof} to derive a similar iteration $(\hat \theta_j)$, that also converges to $\theta^\star$, which however avoids the evaluation of $\EE[G(\theta_j,Z)]$ at most of the steps.
\paragraph{The first recombination step.}
Initialize $\theta_0 \in \R^n$ as before in (S)GD, but before the first step, apply Theorem~\ref{th:tchakalof} to $Z$ and the $n$ coordinate functions of the gradient $z \mapsto G(\theta_0,z)$ to produce a discrete random variable $\hat Z_0$.
By construction, this random variable $\hat Z_0$ has only $n+1$ different possible outcomes and these outcomes are part of the original dataset $\{z_i=(x_i,y_i), i=1,\ldots,n\}$.
Now define the first descent step
\begin{align}
 \hat \theta_{1}:= \theta_0 - \gamma  \EE[G(\theta_0, \hat Z_0)].
\end{align}
Since by construction, $\EE[G(\theta_0,Z)] = \EE[G(\theta_0,\hat Z_0) ]$, it follows that $\hat \theta_1 = \theta_1$.
In general $\EE[G(\hat\theta_1,Z)] \neq \EE[G(\hat \theta_1,\hat Z_0)]$ but the intuiton is that $\EE[G(\hat \theta_1,\hat Z_0)]$ is a good approximation of $ \EE[G(\hat \theta_1, Z) ]$ for reasonable choices of $\gamma$.
Hence, we continue to iterate
\begin{align}\label{eq:gdzro}
 \hat \theta_{j+1}:= \hat \theta_j - \gamma  \EE[G(\hat \theta_j, \hat Z_0)].
\end{align}
until the first time $\tau_1$ a control statistic tells us that the gradient error has become too large.
\paragraph{A control statistic.}
Let $L$ be convex, twice differentiable, and its gradient be Lipschitz, we show later that
a natural choice for control statistic is the quantity
\begin{align}\label{eq:Delta_def}
\Delta_{j,0}:=\EE[G(\theta_{0},\hat{Z}_{0})]\cdot(\hat{\theta}_{j}-\theta_{0})+\frac{c}{2}\|\hat{\theta}_{j}-\theta_{0}\|^{2},
\end{align}
where $c$ is such that %
$v^\top H(\theta,z) v \le c$ for every $v \in \R^n$; the existence of such a $c$ is justified by the assumptions on $L$.
More precisely, $\Delta_{j,0}< \Delta_{j-1,0}$  guarantees that the loss function $L$ continues to decrease.
Hence, we follow the iteration~\eqref{eq:gdzro} until $\Delta_{j,0}\geq \Delta_{j-1,0}$, that is until step $\tau_1:=\inf \{j>0: \Delta_{j,0}\geq \Delta_{j-1,0}\}$, where we fix $\Delta_{0,0}:=0$.
At time $\tau_1$ we then simply update $\hat Z_0$ to $\hat Z_1$ so that the gradients are matched at the point $\hat\theta_{\tau_1-1}$,
that is $\hat Z_1$ is such that $\EE[G(\hat \theta_{\tau_{1}-1},Z)]=\EE[G(\hat \theta_{\tau_1-1},\hat Z_1)]$, and then we continue as before.

\paragraph{CaGD in a nutshell.}
To sum up, we set $\tau_0:=0$, $\Delta_{0,0}=0$, and construct $\hat Z_0$ such that $\EE[G(\theta_0,Z)]=\EE[G(\theta_0,\hat Z_0)]$.
We then update, for $j\geq 0$,
\begin{align}
  \hat \theta_{j+1}:= \hat \theta_j- \gamma  \EE[G(\hat \theta_j , \hat Z_0)] \text{ as long as } \Delta_{j,0}<\Delta_{j-1,0}.
\end{align}
At time $\tau_1$ %
we compute $\hat Z_1$ such that
\begin{align}
 \EE[G(\hat \theta_{\tau_1-1}, Z)]= \EE[G(\hat \theta_{\tau_1-1}, \hat Z_1)]
\end{align}
and update for $j \ge \tau_1-1$
\begin{align}
   \hat \theta_{j+1}:= \hat \theta_j- \gamma  \EE[G(\hat \theta_j, \hat Z_1)], \text{ as long as } \Delta_{j,1}< \Delta_{j-1,1}
\end{align}
where $\Delta_{j,1}:=\EE[G(\hat{\theta}_{\tau_{1}-1},\hat{Z}_{1})]\cdot(\hat{\theta}_{j}-\hat{\theta}_{\tau_{1}-1})+\frac{c}{2}\|\hat{\theta}_{j}-\hat{\theta}_{\tau_{1}-1}\|^{2}$ and $\Delta_{\tau_1-1,1}=0$.
At time $\tau_2:=\inf \{j>\tau_1: \Delta_{j,1} \ge \Delta_{j-1,1}\}$ we compute $\hat Z_2$ such that $ \EE[G(\hat \theta_{\tau_2-1}, Z)]= \EE[G(\hat \theta_{\tau_2-1}, \hat Z_2)]$, etc.

\subsubsection{Convergence and convergence rate} 
The above is the main structure of our first algorithm, denoted as 
Carath\'eodory Gradient descent (CaGD).
However, we add three further modifications. First, we stop as soon as the gradient or the value of the loss function is smaller than a given $\epsilon$ since this means we are already close enough to the minimum;
second, we bound the number of iterations between two recombinations by a constant, that is $\tau_{k+1}-\tau_k \le \operatorname{it\_max\_Ca} $, to {avoid pathological cases, see Theorem~\ref{th:general_conv} for more details};
third, we allow to match a general oracle direction $D_j$ at step $j$.
The choice $D_j=-\EE[G(\hat \theta_j, Z)]$ is the most relevant for this section, but the general oracle formulation allows to use more involved choices; e.g. momentum strategies in Section \ref{sec:BCD}.
This leads to Algorithm~\ref{algo:Ca-accel}.
In Algorithm \ref{algo:Ca-accel} we write $D_j(\{\theta\},Z)$ to express its dependencies on the data $Z$ and the sequence $\{ \theta\}$ computed
up to the step $j$, although it could depend also on the loss function $L$, in particular it could depend on its derivatives $G$, $H$, etc.

Theorem~\ref{th:general_conv} shows that it converges whenever we match oracle directions.%

\begin{algorithm}\caption{Carath\'eodory Sampling Acceleration}\label{algo:Ca-accel}
\begin{algorithmic}[1]
  \State{Initialize $\hat \theta_0 $}
	\State{$j \gets 1$, $k \gets 0$}\Comment{$j$ counts steps, $k+1$ the number of recombinations}
  \State{$\tau_0 \gets 0$} \Comment{$\tau_k$ is the step we made with the $(k+1)$th recombination}
  \State{{$\Delta_{\tau_0,0} \gets 0$}}
  \State{{$\operatorname{Grad}_0\gets  \EE[G(\hat \theta_{\tau_{0}}, Z)]$}}
     \While{{($\|\operatorname{Grad}_{\tau_k} \|>\epsilon_1$\textbf{ or }$|L(\hat\theta_{\tau_k},Z) |>\epsilon_2$)} \textbf{and} $j\leq$ it\_max }
     \State{{Compute $\hat Z_k$ such that $\ensuremath{\EE[D_{\tau_{k}}(\{\hat{\theta}\},\hat{Z}_{k})]=\EE[D_{\tau_{k}}(\{\hat{\theta}\},Z)]}$}}\label{step:compute_meas}
     \Comment{%
     Reduce $Z$ to $\hat Z_k$}
     \While{$\Delta_{j,k}<\Delta_{j-1,k}$ \textbf{and} $j-\tau_k\leq$ it\_max\_Ca }
     \State{$\hat \theta_{j}\gets \hat \theta_{j-1}+\gamma\EE[D_{j-1}(\{\hat \theta\},\hat Z_{k})]$  }\label{step:compute_reduced_gr}
     \State{$\Delta_{j,k} \gets \operatorname{Grad}_{\tau_k} \cdot (\hat{\theta}_{j}-\hat{\theta}_{\tau_{k}}) + \frac{c}{2} \|\hat{\theta}_{j} - \hat{\theta}_{\tau_{k}}\|^{2}$}\label{step:statistic_control}
     \State{$j\gets j+1$ }%
     \EndWhile
     \If{$j-\tau_k \not= \operatorname{it\_max\_Ca}$ }\label{step:strange_if}
       \State{$\tau_k, j \gets j-1$}
     \Else \State{$\tau_k, j \gets j$}
     \EndIf
     \State{{$\operatorname{Grad}_{\tau_k}\gets  \EE[G(\hat \theta_{\tau_{k}}, Z)]$, $\quad\Delta_{\tau_k,k}\gets0$}}\label{step:full_gradient}
     \State{$k \gets k+1$}
     \EndWhile
     \textbf{ and return} $j$, $\hat \theta_j$
\end{algorithmic}
\end{algorithm}

\begin{theorem}\label{th:general_conv}
{Let $L_\theta$ be convex, twice differentiable in $\theta$ and its gradient $G$ be Lipschitz.
If the quantities $\{\theta_j\}$ defined as
\[
\theta_{j}-\theta_{j-1}=\gamma\EE[D_{j-1}(\{\theta\},Z)]
\]
converge to the minimum $\theta^*$,  i.e.~$\lim_{j\to\infty}\theta_j =\theta^*$, then also the sequence of $\{\hat\theta\}$ computed via Algorithm \ref{algo:Ca-accel} converges to $\theta^*$, $\lim_{j\to\infty}\hat\theta_j =\theta^*$. }
\end{theorem}

\begin{proof}%
Thanks to the hypothesis there exists $c$ s.t.%
\begin{align}
\EE[L(\hat{\theta}_{j},Z)]\!=&\EE[L(\hat{\theta}_{0},Z)]\!+\!\EE[G(\hat{\theta}_{0},Z)]\!\cdot\!(\hat{\theta}_{j}\!-\!\hat{\theta}_{0})\!+\!\dfrac{1}{2}(\hat{\theta}_{j}\!-\!\hat{\theta}_{0})^{\top}\!\!\cdot\!\EE[H(\bar{\theta},Z)]\!\cdot\!(\hat{\theta}_{j}\!-\!\hat{\theta}_{0})\\\le&\EE[L(\hat{\theta}_{0},Z)]+\EE[G(\hat{\theta}_{0},Z)]\cdot(\hat{\theta}_{j}-\hat{\theta}_{0})+\frac{c}{2}\|\hat{\theta}_{j}-\hat{\theta}_{0}\|^{2},\quad\text{ for }j\geq0,
\end{align}
where $\bar\theta$ is a convex combination of $\hat\theta_{j}$ and $\hat\theta_{0}$.
It is now easy to see that we have a condition to check, in order to rebuild the measure:
we update the measure after $\tau_1$ steps, where
\begin{align}
\tau_{1}:=&\inf\{j\geq1:\Delta_{j,0}\geq\Delta_{j-1,0}\}\\
\Delta_{j,0}:=&\EE[G(\hat{\theta}_{0}, Z)]\cdot(\hat{\theta}_{j}-\hat{\theta}_{0})+\frac{c}{2}\|\hat{\theta}_{j}-\hat{\theta}_{0}\|^{2},
\end{align}
where $\Delta_{0,0}=0$.
We have that $\{\Delta_{0,0}, \Delta_{1,0},\ldots\Delta_{\tau_1-1,0}\}$ is a negative decreasing sequence and therefore
\[
\EE[L(\hat{\theta}_{\tau_{1}-1},Z)]\leq\EE[L(\hat{\theta}_{0},Z)].
\]
In particular, note that $\Delta_{1,0}\leq0$, since $\hat{\theta}_{1}=\theta_{1}:=\theta_{0}+\gamma\EE[D_{0}(\{\theta\},Z)]$, thanks to Theorem \ref{th:cath} and the definition of $\hat Z_0$, therefore
 $\tau_1\geq 2$.  $\tau_1-1 = 1 $ means that the reduced r.v. $\hat Z_0$ computed has been useless, i.e. we have done only one step with the reduced measure that we could have done directly using $\EE[D_{0}(\{\theta\},Z)]$ without computing the reduced measure.

The reasoning can be easily generalized: we can define for $k> 1$, and 
$j \geq \tau_{k-1}$
\begin{align}
\Delta_{j,k}:=&\EE[G(\hat{\theta}_{\tau_{k}-1}, Z)]\cdot(\hat{\theta}_{j}-\hat{\theta}_{\tau_{k}-1})+\frac{c}{2}\|\hat{\theta}_{j}-\hat{\theta}_{\tau_{k}-1}\|^{2}\\\tau_{k}:=&\inf\{j\geq\tau_{k-1}:\Delta_{j,k-1}\geq\Delta_{j-1,k-1}\},
\end{align}
where $\Delta_{\tau_{k}-1,k}=0$. 
The proof of the convergence follows since if %
$\tau_k-\tau_{k-1}=2$ we follow the directions $D_j(\{\theta\},Z)$ which converge for the hypothesis, whereas if $\tau_k-\tau_{k-1}\geq 2$ the value of $L$ decreases,
\[
\EE[L(\hat{\theta}_{\tau_{k}-1},Z)]\leq\EE[L(\hat{\theta}_{\tau_{k-1}-1},Z)].
\]
Moreover, to avoid pathological cases, e.g. $\Delta_{1,k}<\Delta_{2,k}<\ldots\searrow-a,$ $a>0$ in which cases $L(\hat\theta_j,Z)$ cannot decrease ``enough'', we impose a number of maximum iterations that the Algorithm can do with the reduced measure.
\qed
\end{proof}
Theorem \ref{th:general_conv} can be easily extended to the case where the learning rate $\gamma$ is not fixed.
Theorem~\ref{th:convergence_CaGD} gives the convergence rate for the choice $D_j=-\EE[G(\hat \theta_j, Z)]$.

\begin{theorem}\label{th:convergence_CaGD}
{Let $L_\theta$ be convex, twice differentiable in $\theta$ and its gradient $G$ be Lipschitz.} Then if $D_j=-G(\hat\theta_j,Z)$, Algorithm \ref{algo:Ca-accel} converges to $\theta^*$,
and its convergence rate is
\begin{align}\label{eq:Ca-GDcomplexity}
|L(\hat\theta_{j})-L(\theta^{*})|\leq\frac{1}{2\gamma} J  \frac{\|\hat\theta_{0}-\theta^{*}\|^{2}}{j},
\end{align}
where $j$ is the number of iterations, and $J$ is the number of times the reduced measure is used
(as per Algorithm \ref{algo:Ca-accel}, we can conservatively bound $J<\operatorname{it\_max\_Ca}$).
\end{theorem}

\begin{proof}%
The convergence is a simple application of Theorem \ref{th:general_conv}.
We can show that Algorithm \ref{algo:Ca-accel} does not reduce the order of convergence of the standard GD. Let us call $\hat\theta_i$ the
sequence of weights obtained by Algorithm \ref{algo:Ca-accel} in chronological order %
\begin{align}
\{\hat{\theta}_{0},\hat{\theta}_{1},\ldots,\hat{\theta}_{\tau_{1}-1},\hat{\theta}_{\tau_{1}},\ldots,\hat{\theta}_{\tau_{2}-1},\hat{\theta}_{\tau_{2}},\ldots\},
\end{align}
where for $k>1$ ($k=1$) $\tau_k$ indicates the number of times we use the reduced measure computed using $\theta_{\tau_{k-1}-1}$ ($\theta_{0}$). Moreover, let us suppose that for any step $j$ we have a map $S$ that tell us
the step where we had recomputed the measure the last time, so $
S(j)=\max\{k: \tau_k\leq j\}$.
Let us recall that if the function is convex we have that
\[
L(\theta)\leq L(\theta^{*})+\nabla L(\theta)(\theta-\theta^{*})
\]
where $\theta^{*}$ is the minimum, moreover if $\{\theta_i\}$ are the weights computed using the standard GD, we can say that
\[
L(\theta_{i+1})\leq L(\theta_{i})-\frac{1}{2}\gamma\|\nabla L(\theta_{i})\|^{2},
\]
if $\gamma$ respects the usual conditions, i.e. $\gamma\leq 1/\text{Lip}(L)$, where $\text{Lip}(L)$ indicates the Lipschitz constant of $L$.
We know that $L(\hat{\theta}_{j})\leq L(\hat{\theta}_{\tau_{S(j)}-1})+\Delta_{j,\,S(j)}$
therefore, since $\Delta_{j,\,S(j)}\leq 0$%
\begin{align}
L(\hat{\theta}_{j})\leq L(\hat{\theta}_{\tau_{S(j)}})\leq L(\theta^{*})+\nabla L(\hat{\theta}_{\tau_{S(j)}-1})(\hat{\theta}_{\tau_{S(j)}-1}-\theta^{*})-\frac{1}{2}\gamma\|\nabla L(\hat{\theta}_{\tau_{S(j)}-1})\|^{2},
\end{align}
which rearranging the terms and using that $\hat{\theta}_{\tau_{S(j)}}-\hat{\theta}_{\tau_{S(j)}-1}=\EE[G(\hat{\theta}_{\tau_{S(j)}-1}, Z)]=\EE[G(\hat{\theta}_{\tau_{S(j)}-1},\hat{Z}_{\tau_{S(j)}-1})]$ becomes
\begin{align}
L(\hat{\theta}_{j})-L(\theta^{*})\leq\frac{1}{2\gamma}\left(\|\hat{\theta}_{\tau_{S(j)}-1}-\theta^{*}\|^{2}-\|\hat{\theta}_{\tau_{S(j)}}-\theta^{*}\|^{2}\right).
\end{align}
Thus, %
\begin{align}
\sum_{l=1}^{j}L(\hat{\theta}_{l})-L(\theta^{*})\leq&\frac{1}{2\gamma}\sum_{l=1}^{j}\left(\|\hat{\theta}_{\tau_{S(l)}-1}-\theta^{*}\|^{2}-\|\hat{\theta}_{\tau_{S(l)}}-\theta^{*}\|^{2}\right)\\=&\frac{1}{2\gamma}\sum_{k:\tau_{k}\leq j}\left(\tau_{k}-\tau_{k-1}\right)\left(\|\hat{\theta}_{\tau_{k}-1}-\theta^{*}\|^{2}-\|\hat{\theta}_{\tau_{k}}-\theta^{*}\|^{2}\right)\\\leq&\frac{1}{2\gamma}\max_{k:\tau_{k}\leq j}\left\{ \tau_{k}-\tau_{k-1}\right\} \sum_{k:\tau_{k}\leq j}\left(\|\hat{\theta}_{\tau_{k}-1}-\theta^{*}\|^{2}-\|\hat{\theta}_{\tau_{k}}-\theta^{*}\|^{2}\right)\\\leq&\frac{1}{2\gamma}\max_{k:\tau_{k}\leq j}\left\{ \tau_{k}-\tau_{k-1}\right\} \|\hat{\theta}_{0}-\theta^{*}\|^{2}.
\end{align}
Therefore it holds that
\begin{align}
L(\hat{\theta}_{j})-L(\theta^{*})\leq\frac{1}{j}\sum_{l=1}^{j}L(\hat{\theta}_{l})-L(\theta^{*})\leq\frac{1}{2\gamma}\max_{k:\tau_{k}\leq j}\left\{ \tau_{k}-\tau_{k-1}\right\} \frac{\|\hat{\theta}_{j}-\theta^{*}\|^{2}}{j}.
\end{align}
\qed
\end{proof}
First, note that the bound in
Equation~\eqref{eq:Ca-GDcomplexity}
assumes that when we use the reduced measures the objective function does not decrease, thanks to
Equation~\eqref{eq:Delta_def}.
Secondly, note that
the constant $c$ in Equation~\eqref{eq:Delta_def} in practice might be unknown and expensive to compute, and when known it might be quite conservative.
In our implementation we use an approximation of the second derivative,
so that $\Delta_{j,k}$ in Equation~\eqref{eq:Delta_def} becomes
\begin{align}
\Delta_{j,k}:=\ensuremath{\EE[G(\hat{\theta}_{\tau_{k}},\hat{Z}_{k})]\cdot(\hat{\theta}_{j}-\hat{\theta}_{\tau_{k}})+\frac{1}{2}(\hat{\theta}_{j}-\hat{\theta}_{\tau_{k}})^{\top}\cdot\mathcal{H}_{k}\cdot(\hat{\theta}_{j}-\hat{\theta}_{\tau_{k}})},\quad j\ge\tau_{k},
\end{align}
where $ \mathcal{H}_{k}:=\left[\EE[G(\hat{\theta}_{\tau_{k}},\hat{Z}_{k})]-\EE[G(\hat{\theta}_{\tau_{k}-1},Z)]\right]^{\top}\cdot\left[1/(\hat{\theta}_{\tau_{k}}-\hat{\theta}_{\tau_{k}-1})\right] $
and
$[1/x]$ denotes a vector whose elements are the reciprocals of those in $[x]$.
To compute the terms $\Delta_{(\cdot,\cdot)}$ we modify Algorithm \ref{algo:Ca-accel} doing two iterations where $\EE [G(\theta,Z)]$ is computed -- see Algorithms \ref{algo:CaBCD_GS} %
and \ref{algo:CaBCD_random} %
in Section~\ref{sec:BCD}.
We do not discuss how to optimally select $\gamma$,  since there exists a broad literature about the optimal selection of the step~\cite{Nesterov2018}.

\subsection{A first comparison of CaGD and GD: logistic regression}\label{sec:exp_GD}
Already comparing the complexity of CaGD, Algorithm \ref{algo:Ca-accel}, to standard GD is not trivial, since the number of recombinations steps is not known a-priori (the times $\tau_k$ specify a recombination compuation depend on the data).
Furthermore, the worst-case complexity of the recombination algorithm itself is much worse than its average complexity, see~\cite{Cosentino2020}.
Hence, the intuition remains that for a sufficiently large number of samples $N$ and low-dimensional $\theta$, the total cost of computing the recombinations is negligible compared to evaluating the full gradient in each descent step.
We present three numerical experiments to test this intuition.
We use classic logistic regression for binary classification (it is easy to check that the assumptions of Theorem \ref{th:convergence_CaGD} are fulfilled in this case, see \cite[Exercise 8.3]{Murphy2012}) and use synthetic data which allows to study various regimes of $N$.
We run both GD and CaGD until either the norm of the gradient is less than \num{1e-3}, or the number of iterations is greater than \num{1e4}.

\begin{figure}[hbt!]
    \centering
         \includegraphics[height=3.3cm,width=0.32\textwidth]{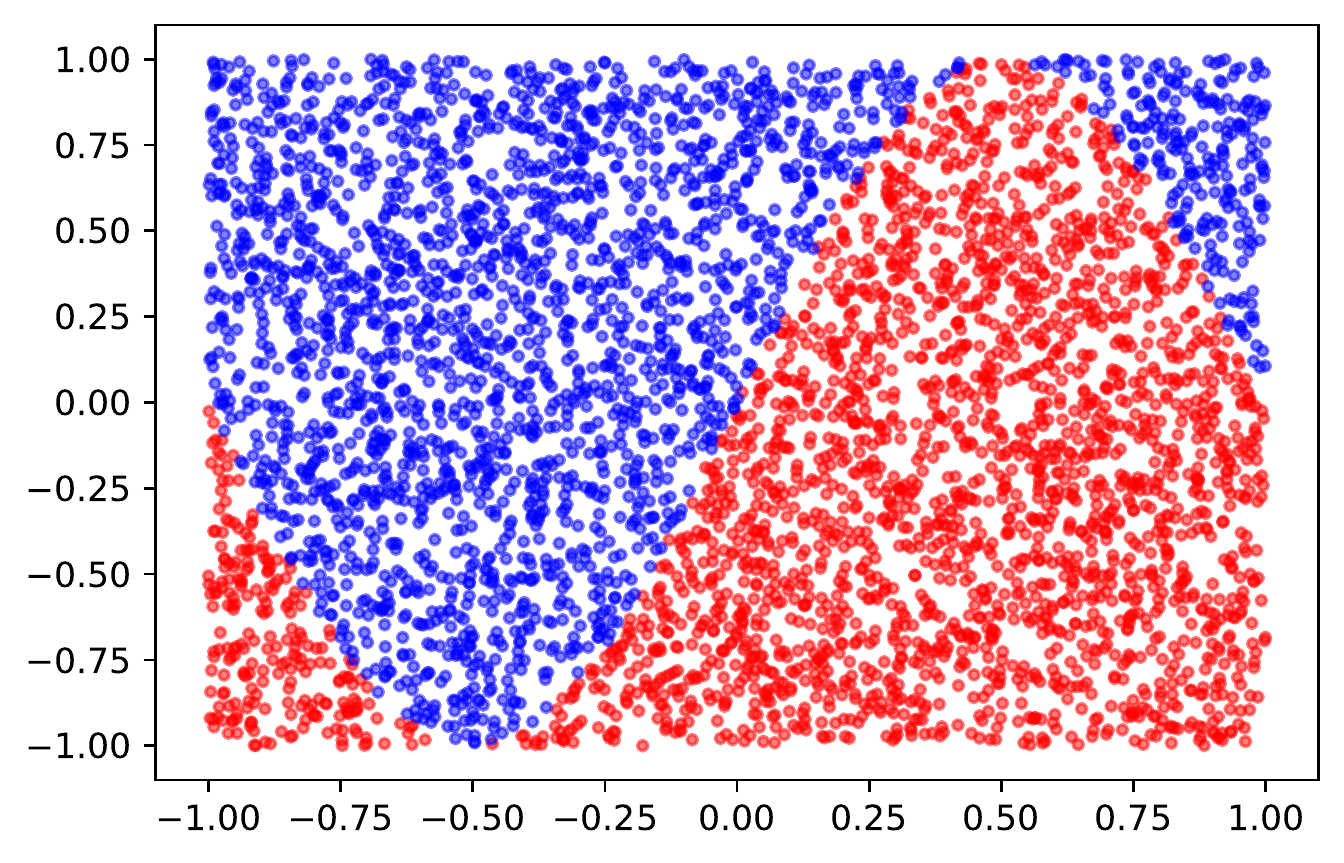}
         \includegraphics[height=3.3cm,width=0.32\textwidth]{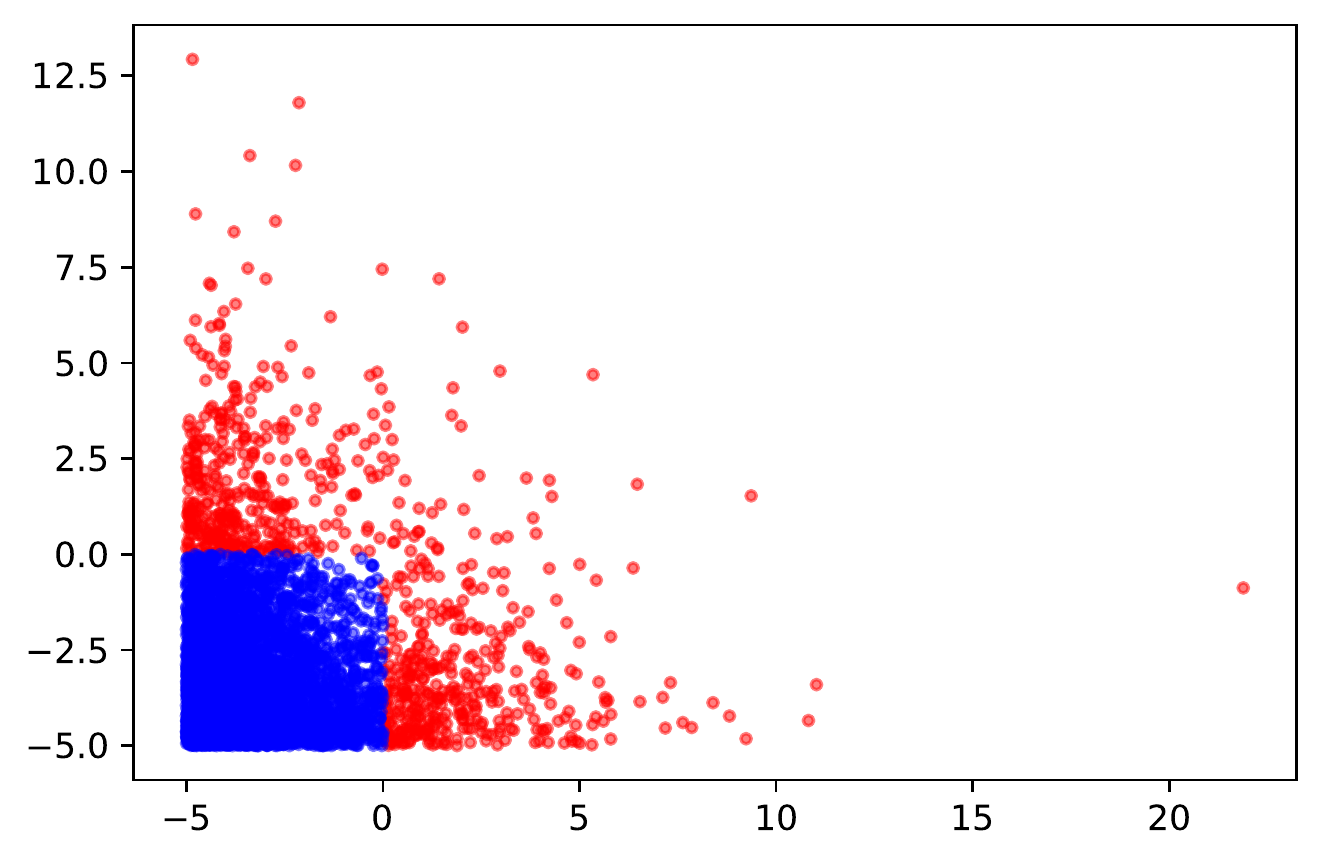}
         \includegraphics[height=3.3cm,width=0.32\textwidth]{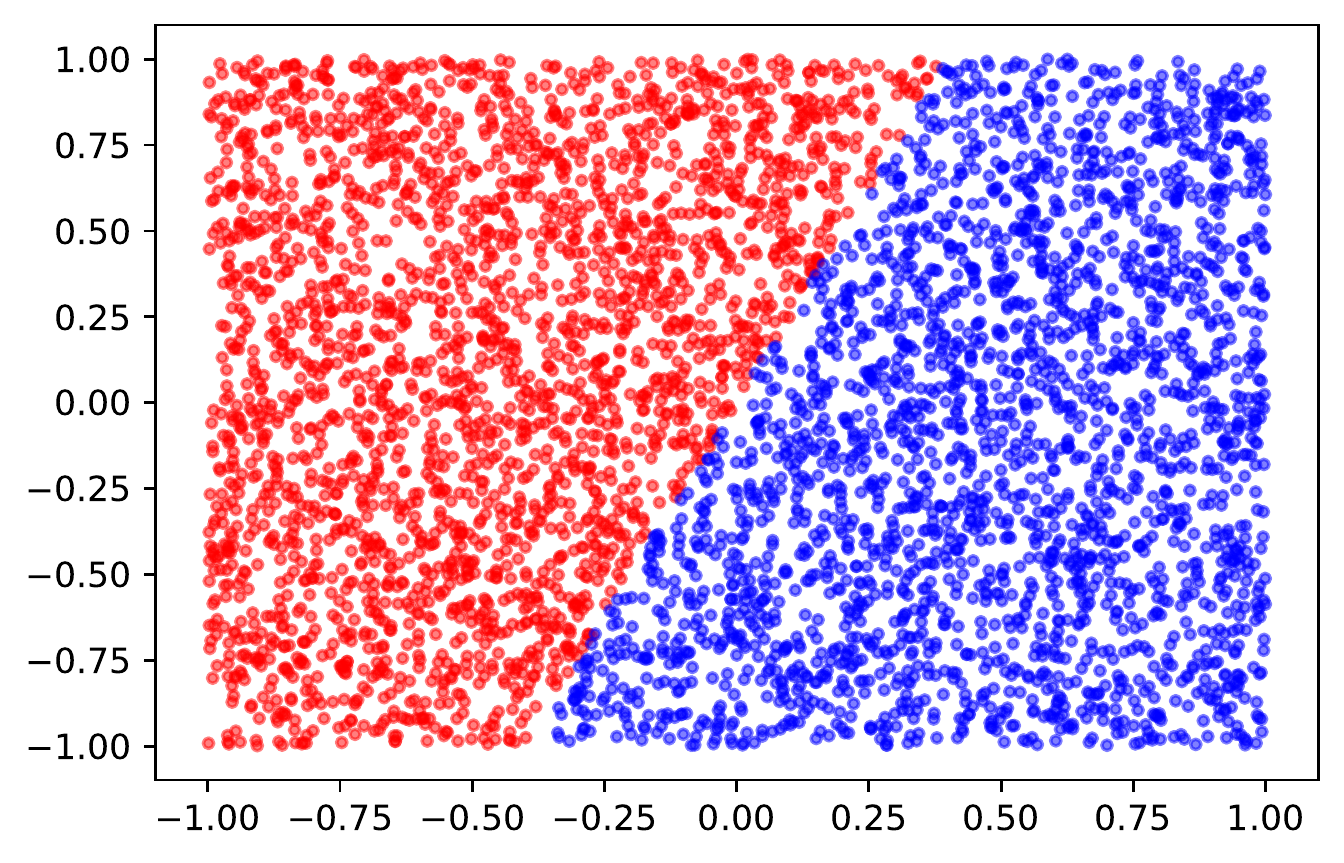}

        \includegraphics[height=3.3cm,width=0.315\textwidth]{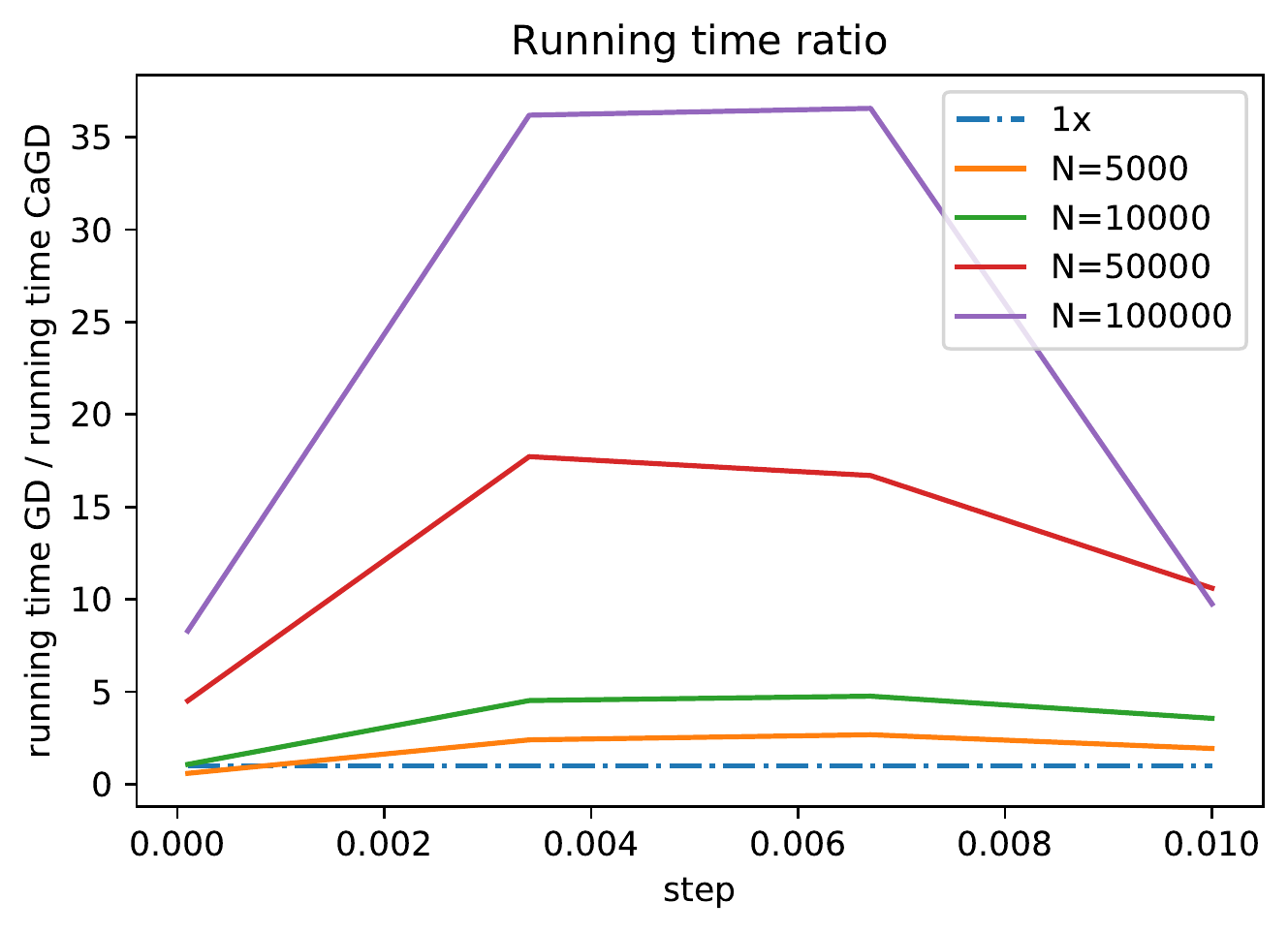}\hspace{1.5mm}
        \includegraphics[height=3.3cm,width=0.315\textwidth]{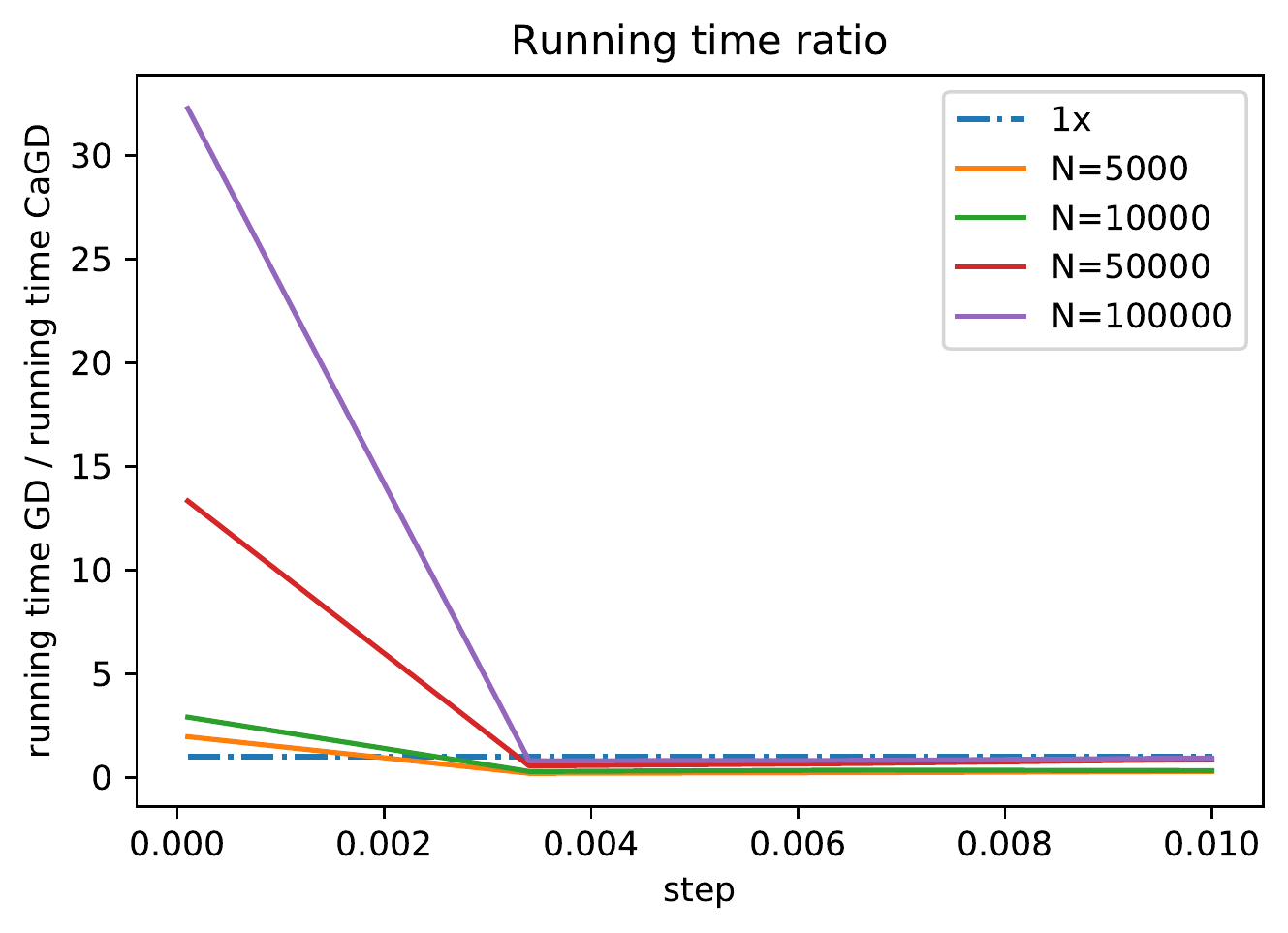}\hspace{1.5mm}
        \includegraphics[height=3.3cm,width=0.315\textwidth]{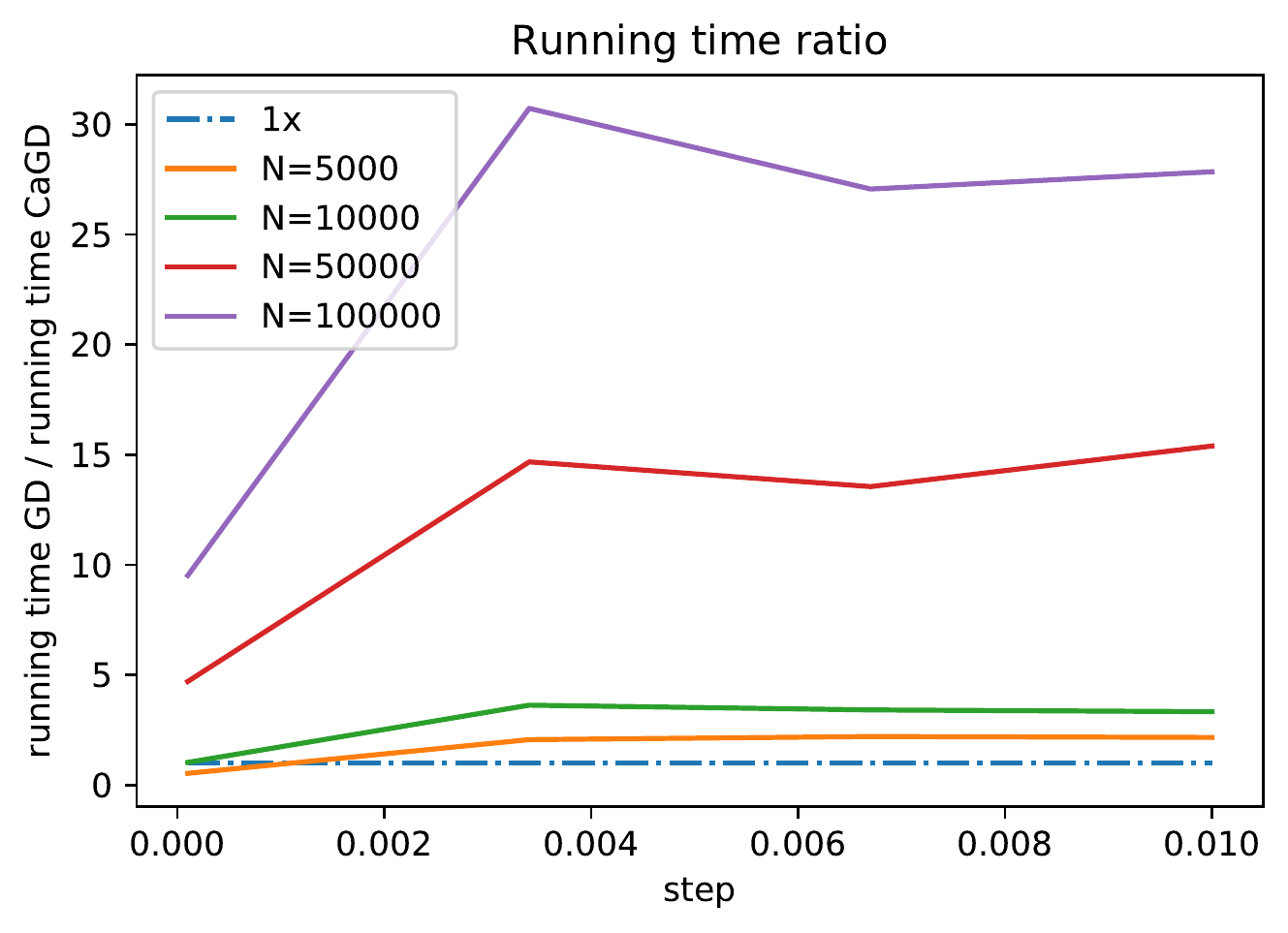}

        \caption{
          Synthetic data generated by sampling from
          (i) a uniform distribution, classified with a sine function,
          (ii) a shifted exponential, classified as $1$ in the third octant, and $0$ otherwise,
          (iii) a uniform distribution, classified with a logistic model with parameter $(-5;2)$.
          The top row shows samples with $N = \num{5000}$.
 The bottom row shows the ratios of running times between standard GD and CaGD, as a function of the step size for various sample sizes $N$.
    }
    \label{fig:2d}
\end{figure}
The results are shown in Figure \ref{fig:2d} and indicate that the improvement in the running time is up to $35$-fold.
Generally the improvement increases as the number of points $N$ increases and when the step size is small (the step of the GD must be small enough for the algorithm to converge to a minimum).
{Another advantageous observation is that  CaGD reaches lower gradient values than GD, because in Algorithm \ref{algo:Ca-accel} the ``true'' gradient $\EE[G(\theta,Z)]$ is only computed at the step~\ref{step:full_gradient}, but we modify $\hat\theta_j$ at step~\ref{step:compute_reduced_gr}.}
In these instances we have employed it\_max\_Ca$=\max\{10/\text{step},10^{\,4}\}$.

Of course, the real benchmarks are SGD variants like SAG and ADAM.
However, CaGD in its simple form above is not competitive to such SGD variants since its computational bottleneck is that the recombination step scales cubically in the dimension $n$ of $\theta$ which makes it infeasible for many datasets.   
In Section~\ref{sec:BCD} below we combine CaGD with BCD to resolve this issues.
This then results in a competitive alternative to SGD variants like SAG and ADAM.
\section{Carath\'eodory Block Coordinate Descent}\label{sec:BCD}
The computational bottleneck of CaGD is in step~\ref{step:compute_reduced_gr}, where a recombination algorithm is run to compute $\hat Z$.
As mentioned before, the recombination algorithm \cite{maria2016a} has the worst case complexity $O(Nn+n^3\log(N/n))$ where $n$ is the dimension of $\theta$ and $N$ the number of samples.
However, unlike the other recombination algorithms, the greedy algorithm in \cite{Cosentino2020} is designed to have a much lower average complexity in $N$. 
Unfortunately, in terms of the dimension $n$, it suffers like all the other recombination algorithms from the cubic term $n^3$ which makes CaGD not suitable for high-dimensional problems.
In this Section, we combine CaGD with Block Coordinate Descent (BCD) \cite{Nesterov2012,Wright2015,Richtarik2016,Nutini2017,Beck2013,Csiba2017} to resolve this cubic complexity.
This leverages the strengths of BCD in terms of compuational efficiency and of CaGD in terms of low variance estimators for the expected gradient estimate.

\subsection{From Block Coordinate Descent to Carath\'eodory Block Coordinate Descent}
BCD selects in every descent step a small subset of the $n$ coordinates of $\theta$. This, in turn allows us to apply the recombination algorithm to a small number of coordinates, typically $n$ is between $2$ and $5$ so that the cubic complexity in number of coordinates becomes negligible.
The core idea of BCD is to update only a subset of the coordinates at every step,%
\begin{align}
\theta_{j+1} = \theta_{j}-\gamma \EE[G^{(B({j}))} (\theta,Z)],
\end{align}
where ${B}({j})$ denotes a set of coordinates, and $G^{(B({j}))}$ denotes the function that returns the gradient for the coordinate in ${B}({j})$,
and sets it to be equal to $0$ for the other coordinates.
If the problem we want to solve is of the form
\begin{align}
\min_\theta \, L_\theta = \min_\theta (f(\theta)+g(\theta)),
\end{align}
where $f$ is convex and $g$ is (block) separable, i.e.~$g=\sum_{m=1}^b g_m$, $g_m:\R^{n_m}\to\R$ and $\sum_m n_m \leq n$, $b\leq n$,
then BCD converges to a minimum and the rate of convergence is that of standard GD, up to some constants depending on different factors, e.g.~the number of directions we update at any step, the strategy to choose such directions, the separability of the objective function $g$; see~\cite{Wright2015,Richtarik2016} for a detailed study.
Notable examples of optimisation problems with functions $(f,g)$ abiding by the previous condition are least-squares problems with LASSO regularisation~\cite{Richtarik2016,Nesterov2012,Fercoq2015}, which is where BCD appears to be more effective (we analyse this problem in the next subsection).
A well-studied aspect of BCD is its parallelisation~\cite{Fercoq2015,Liu2015,Bradley2011,You2016,Scherrer2012},
which can be studied in terms of the spectral radius of the data~\cite{Bradley2011}.
In the following, we focus on simple applications of BCD %
to highlight the improvements that are due to CaBCD, rather than BCD optimizations.

\subsection{Update rules} 
An important aspect is how to select the directions
for the descent step~\cite{Nutini2017,Li2015,Richtarik2016a,Dhillon2011,Glasmachers2013,Lei2016,Li2018,Qu2016,Qu2016a,Beck2013}, and how many directions to select: cyclic versus acyclic, deterministic versus random, or via the Gauss-Southwell (GS) rule (discussed below).
The main differences amongst these options hinges on whether or not we can afford to compute the full gradient:
if we can, then the GS rule is expected to be the best strategy;
if we cannot, then the random acyclic strategy seems to perform better than
the deterministic cyclic ones \cite{Lee2018}.
Moreover, these strategies can be implemented in parallel, exploiting multiple processors.
We focus on the following two acyclic strategies, whose respective procedures are presented in Algorithm~\ref{algo:CaBCD_GS} and~\ref{algo:CaBCD_random}; {for a detailed comparison see~\cite{Nutini2015,Lee2018}}:
\begin{itemize}
\item \textit{Modified Gauss-Southwell (GS).}
If we can compute the full gradient,
then we can select directions where the gradient is larger in absolute value.
A rule of thumb is to consider only a percentage of the ``total'' value of the gradient (in our experiments we consider 75\%):
\begin{enumerate}[label=(\roman*)]
\item Let us call $\nabla_S$ the vector with the absolute value of the directions of the GD sorted in descending order, i.e.
$\big|\nabla L^{\left(\nabla_{S}^{(r)}\right)}\big|\geq\big|\nabla L^{\left(\nabla_{S}^{(q)}\right)}\big|$ if $r\leq q$;
\item We consider the directions where the gradient is bigger in absolute value, namely the first $\hat n$ components of $\nabla_S$, where
\[
\hat n := \inf \left\{ q:\sum_{jr=1}^{q} \left|\nabla L^{\left(\nabla_{S}^{(r)}\right)}\right|>\text{Percentage}\right\};
\]
\item We split the $\hat n$ directions in $b=\hat n/s$ blocks of size $s$, respecting the ordering.
\end{enumerate}
\item \textit{Random.}
If we cannot compute the full gradient,
then we can group the directions
into $n/s$ blocks of size $s$,
and perform BCD over the blocks.
In the experiments of the next subsection we randomly group half of the directions per iteration.
{The condition to terminate in Algorithm~\ref{algo:CaBCD_random} depends ``only'' on the loss function $L$ since we cannot compute the full gradient.}
\end{itemize}
In~\cite{Nesterov2012,Fercoq2015}, the selection of the coordinate step $\gamma$ is given as a sub-optimisation problem, which for simplicity we skip in the following. %
Furthermore,~\cite{Nesterov2012} shows that a momentum application to the single block of directions can improve the rate of convergence: knowing the convexity of the function to be optimised,
it is possible to obtain an accelerated method with an improved rate of convergence ${O}(1/j^2)$, where $j$ is the number of iterations, in the same spirit of~\cite{Nesterov1983}.
Our implementation has been done in a synchronous and parallel manner (cf.~discussion above).
\begin{algorithm}\caption{Carath\'eodory BCD - Modified Gauss-Southwell (GS) rule}\label{algo:CaBCD_GS}
\begin{algorithmic}[1]
   \State{Initialize $\hat \theta_0 $}
	\State{$j \gets 1$, $k \gets 0$
	}\Comment{$j$ counts steps,
	$\sum_{l=0}^{k}b_l$ counts the number of recombinations}
    \State{$\operatorname{Grad_0}\gets \EE[G(\hat\theta_{0},Z)] $}
      \While{{($\|\operatorname{Grad}_{j-1} \|>\epsilon_1$\textbf{ or }$|L(\hat\theta_{j-1},Z) |>\epsilon_2$)} \textbf{and} $j\leq$ it\_max }
      \State{$\hat \theta_{j} \gets \hat \theta_{j-1}+ \gamma \EE[ D_{j-1}(\{\hat\theta\},Z)]$  }
      \State{$\operatorname{Grad}_j\gets \EE[G(\hat\theta_{j},Z)] $}
     \State{Build $b_k$ blocks $B(m,k)$, $m=1,...,b_k$ using the $\EE[G(\hat\theta_{j},Z)] $ and the GS rule}
      \State{$\tau_k\gets j$}%
      \State{$ j\gets j+1$}
     \For{%
     $m=1,...,b_k$,\textbf{ in parallel}}%
     \State{$j_m\gets j$, $\quad \Delta_{\tau_k,k}^{(m)}\gets 0$}\Comment{$j_m-1=\tau_k$}
     \State{$
     \operatorname{Hessian}_{k}^{(m)}\gets\left[\operatorname{Grad}_{\tau_{k}}^{(m)}-\operatorname{Grad}_{\tau_{k}-1}^{(m)}\right]^{\top}\cdot\left[1/(\hat{\theta}_{\tau_{k}}^{(m)}-\hat{\theta}_{\tau_{k}-1}^{(m)})\right]$
     }
     \State{Compute $\hat{Z}_{k}^{(m)}$ s.t. $ \EE\left[D_{\tau_{k}}^{(m)}\left(\{\hat{\theta}\},\hat{Z}_{k}^{(m)}\right)\right]=\EE\left[D_{\tau_{k}}^{(m)}\left(\{\hat{\theta}\},Z\right)\right]$}
     \While{$\Delta_{j_m,k}^{(m)}\leq\Delta_{j_m-1,k}^{(m)}$ %
     \textbf{and} $j_m-\tau_k\leq$it\_max\_Ca }
     \State{$\hat{\theta}_{j_m}^{(m)}\gets\hat{\theta}_{j_m-1}^{(m)}+\EE\left[D_{j_m-1}^{(m)}\left(\{\hat\theta\},\hat{Z}_{k}^{(m)}\right)\right]$  }
     \State{$\delta_{j_m,k}^{(m)}\gets \hat{\theta}_{j_{m}}^{(m)}-\hat{\theta}_{\tau_{k}}^{(m)}$}
     \State{$\Delta_{j_{m},k}^{(m)}\gets\operatorname{Grad}_{\tau_{k}}^{(m)}\cdot\delta_{j_{m},k}^{(m)}+\left(\delta_{j_{m},k}^{(m)}\right)^{\top}\cdot\operatorname{Hessian}_{k}^{(m)}\cdot\delta_{j_{m},k}^{(m)}$}
     \State{$j_m\gets j_m+1$}
     \EndWhile
     \If{$j_m-\tau_k \not= \operatorname{it\_max\_Ca}$ }
       \State{$\tau_{m,k+1} \gets j_m-1$} \Comment{{$\tau_{m,k+1}-\tau_k$ steps in $(k+1)$th recombination relative to $B(m,j)$}}
     \Else \State{$\tau_{m,k+1} \gets j_m$}
     \EndIf
     \EndFor
     \State{$j\gets j+\sum_m \tau_{m,k+1}$}
     \State{$ \hat\theta^{(m)}_j \gets \hat\theta^{(m)}_{\tau_{m,k+1}}, \quad \forall m$
     }\Comment{synchronise and update $\hat\theta$}
     \State{$\operatorname{Grad}_j\gets \EE[G(\hat\theta_{j},Z)] $}
     \State{$k\gets k+1, \quad$ $j\gets j+1$}
     \EndWhile
     \textbf{ and return} $j$, $\hat \theta_j$
\end{algorithmic}
{We write $\cdot^{(m)}$ in place of $\cdot^{(B(m,k))}$ to indicate the restriction to the components in the blocks $B(m,k)$.}
\end{algorithm}

\begin{algorithm}\caption{Carath\'eodory BCD - Random}\label{algo:CaBCD_random}
\begin{algorithmic}[1]
   \State{Initialize $\hat \theta_0 $}
	\State{$j \gets 1$, $k \gets 0$
	}\Comment{$j$ counts steps,
	$\sum_{l=0}^{k}b_l$ counts the number of recombinations}
     \While{$|L(\theta_{j-1},Z)|>\epsilon$ \textbf{and} $j\leq$it\_max }
     \State{Build $b$ blocks $B(m,k)$, $m=1,...,b$ using the \textit{Random} rule}
\For{$m=1,...,b$,\textbf{ in parallel}}%
	\State{$\operatorname{Grad}_{j-1}^{(m)}\gets \EE[G^{(m)}(\hat\theta_{j-1},Z)] $}
      \State{$\hat \theta^{(m)}_{j} \gets \hat \theta^{(m)}_{j-1}+ \gamma \EE[ D^{(m)}_{j-1}(\{\hat\theta\},Z)]$  }
      \State{$\operatorname{Grad}_j^{(m)}\gets \EE[G^{(m)}(\hat\theta_{j},Z)] $}
      \State{$\tau_k\gets j$}
     \State{$j_m\gets j+1$, $\quad \Delta_{\tau_k,k}^{(m)}\gets 0$}\Comment{$j_m-1=\tau_k$}
     \State{$
     \operatorname{Hessian}_{k}^{(m)}\gets\left[\operatorname{Grad}_{\tau_{k}}^{(m)}-\operatorname{Grad}_{\tau_{k}-1}^{(m)}\right]^{\top}\cdot\left[1/(\hat{\theta}_{\tau_{k}}^{(m)}-\hat{\theta}_{\tau_{k}-1}^{(m)})\right]$
     }
     \State{Compute $\hat{Z}_{k}^{(m)}$ s.t. $ \EE\left[D_{\tau_{k}}^{(m)}\left(\{\hat{\theta}\},\hat{Z}_{k}^{(m)}\right)\right]=\EE\left[D_{\tau_{k}}^{(m)}\left(\{\hat{\theta}\},Z\right)\right]$}
     \While{$\Delta_{j_m,k}^{(m)}\leq\Delta_{j_m-1,k}^{(m)}$ %
     \textbf{and} $j_m-\tau_k\leq$it\_max\_Ca }
     \State{$\hat{\theta}_{j_m}^{(m)}\gets\hat{\theta}_{j_m-1}^{(m)}+\EE\left[D_{j_m-1}^{(m)}\left(\{\hat\theta\},\hat{Z}_{k}^{(m)}\right)\right]$  }
     \State{$\delta_{j_m,k}^{(m)}\gets \hat{\theta}_{j_{m}}^{(m)}-\hat{\theta}_{\tau_{k}}^{(m)}$}
     \State{$\Delta_{j_{m},k}^{(m)}\gets\operatorname{Grad}_{\tau_{k}}^{(m)}\cdot\delta_{j_{m},k}^{(m)}+\left(\delta_{j_{m},k}^{(m)}\right)^{\top}\cdot\operatorname{Hessian}_{k}^{(m)}\cdot\delta_{j_{m},k}^{(m)}$}
     \State{$j_m\gets j_m+1$}
     \EndWhile
     \If{$j_m-\tau_k \not= \operatorname{it\_max\_Ca}$ }
       \State{$\tau_{m,k+1} \gets j_m-1\quad\quad\quad$} \Comment{{$\tau_{m,k+1}-\tau_k$ steps in $(k+1)$th recombination relative to $B(m,j)$}}
     \Else \State{$\tau_{m,k+1} \gets j_m$}
     \EndIf
     \EndFor
     \State{$j\gets j+\sum_m \tau_{m,k+1}$}
     \State{$ \hat\theta^{(m)}_j \gets \hat\theta^{(m)}_{\tau_{m,k+1}}, \quad \forall m$
     }\Comment{synchronise and update $\hat\theta$}
     \State{$k\gets k+1, \quad$ $j\gets j+1$}
     \EndWhile
     \textbf{ and return} $j$, $\hat \theta_j$
\end{algorithmic}
We write $\cdot^{(m)}$ in place of $\cdot^{(B(m,k))}$ to indicate the restriction to the components of $\cdot$ in the blocks $B(m,k)$.
\end{algorithm}

\begin{figure}[hbt!]
  \centering
  \includegraphics[height=3cm,width=0.32\textwidth]{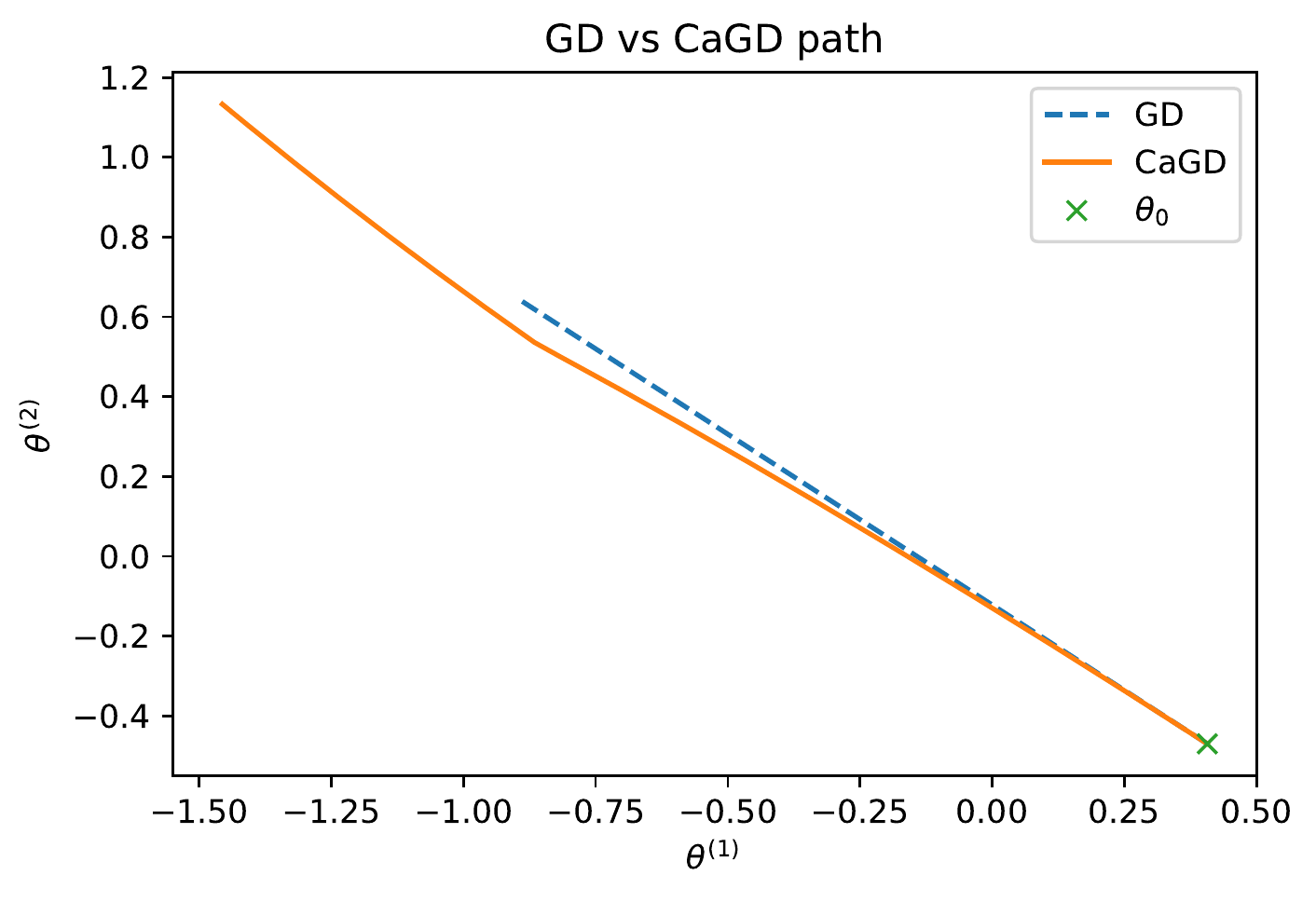}
        \includegraphics[height=3cm,width=0.32\textwidth]{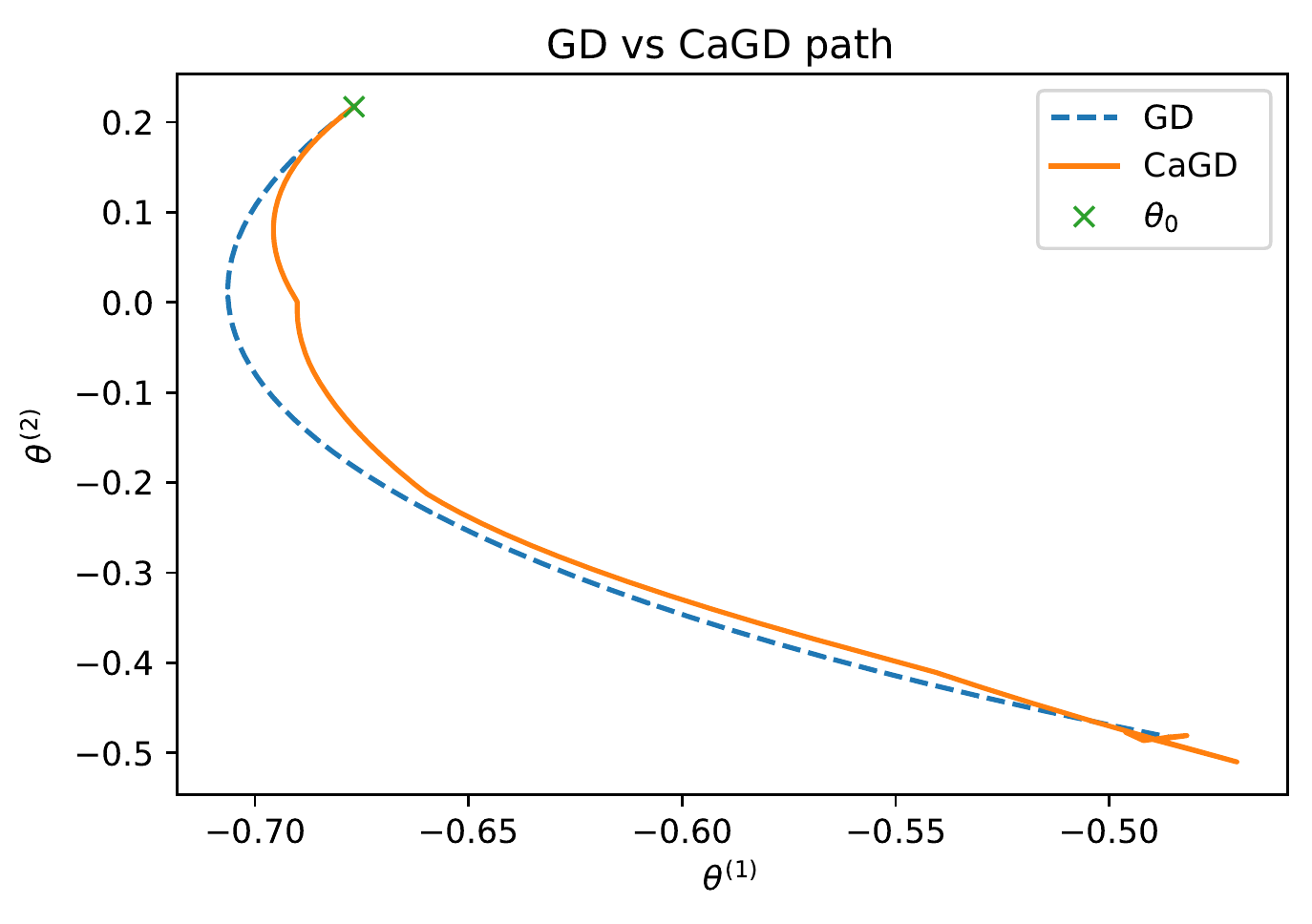}
        \includegraphics[height=3cm,width=0.32\textwidth]{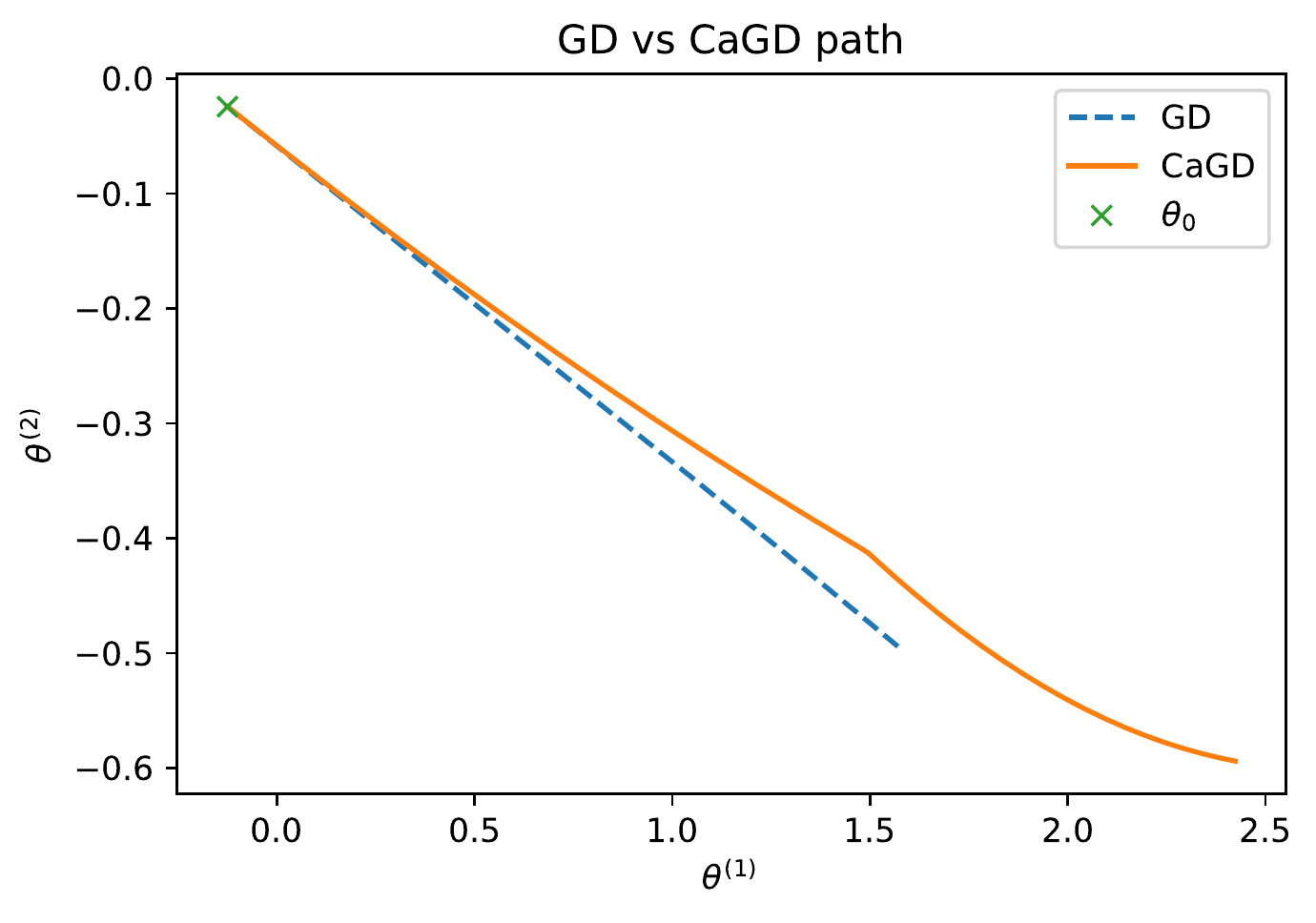}
  \caption{Paths generated by CaGD (Theorem \ref{th:convergence_CaGD}) and GD, for the experiments of Figure~\ref{fig:2d}, same order.}
\end{figure}
\begin{figure}[h!]%
\centering
        \includegraphics[height=3cm,width=0.32\textwidth]{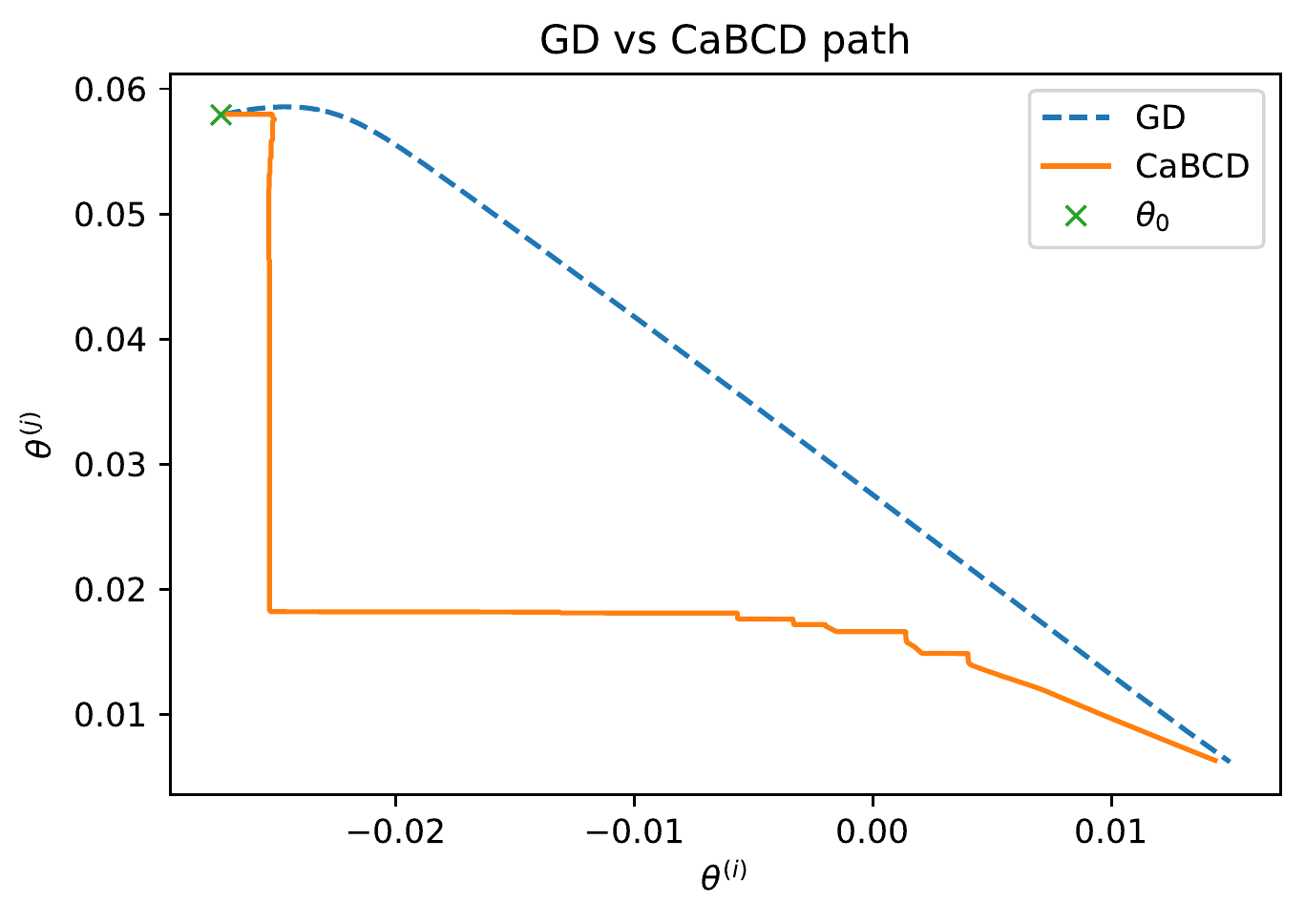}
        \includegraphics[height=3cm,width=0.32\textwidth]{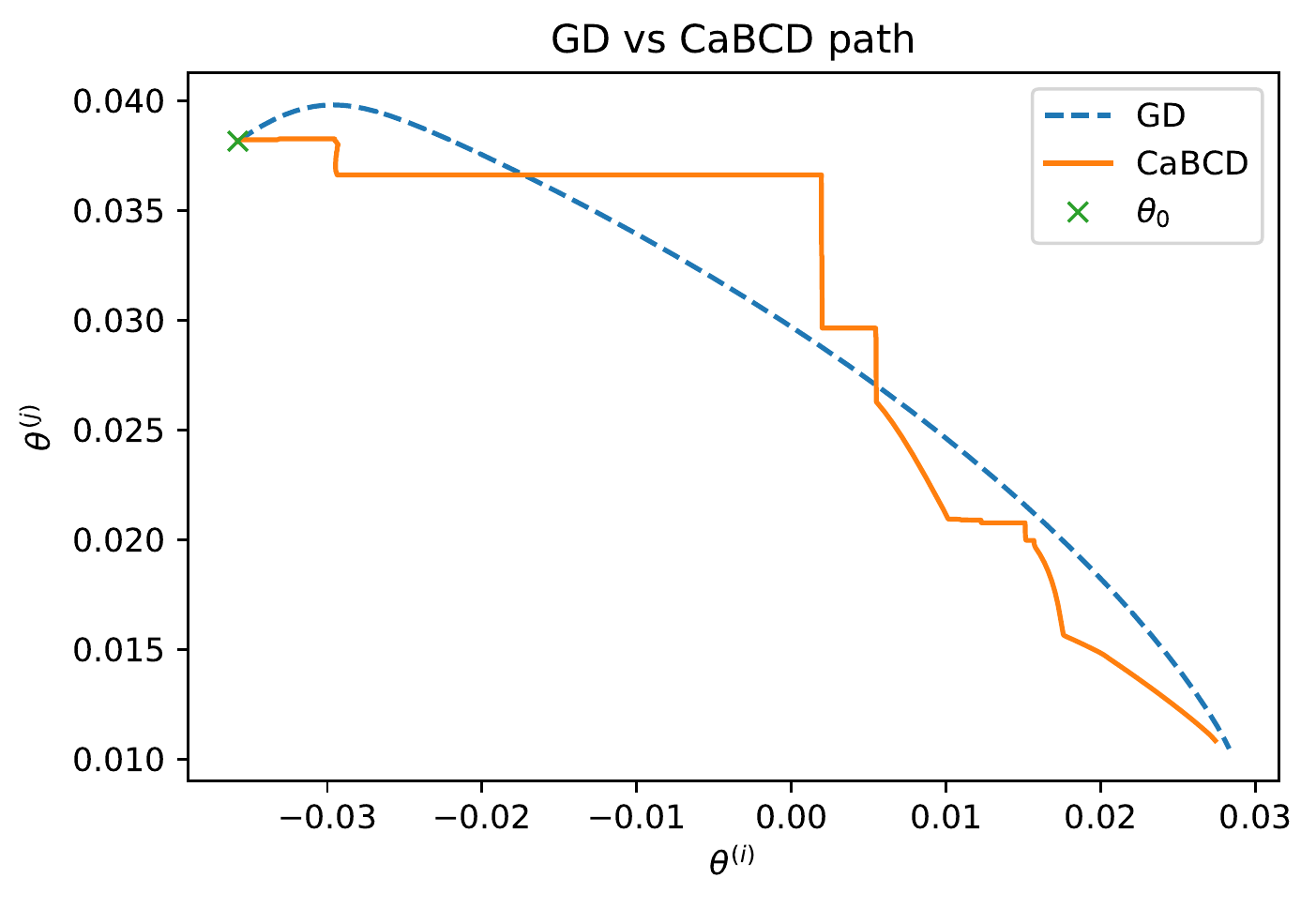}
        \includegraphics[height=3cm,width=0.32\textwidth]{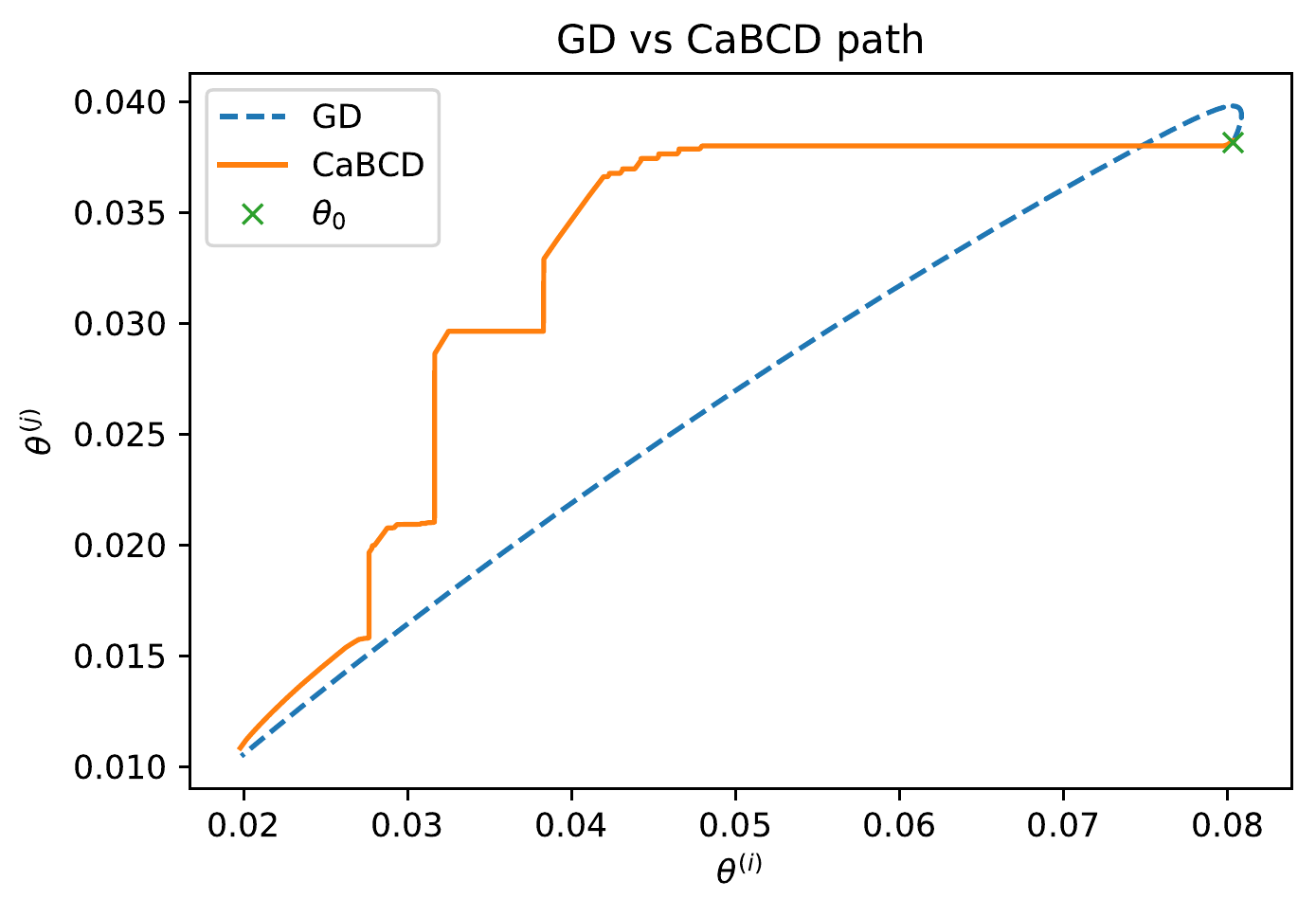}
    \caption{Samples of trajectories followed by the GD and the CaBCD over the parameter space.
    The dotted blue trajectories and the continuous orange trajectories converge to the same desired minimum, though via different paths.
    The CaBCD, between change of directions, uses only a subset of the total points $N$, namely $s+1$ if the size of the selected block is $s$. This Figure has been obtained using the data of Figure~\ref{fig:2d} (center) in the multi-dimensional case.}
    \label{fig:parameter-path}
\end{figure}
\subsection{Experiments: CaBCD vs ADAM vs SAG for LASSO}\label{sec:exp_BCD}
We consider a least-squares problems with LASSO regularisation, i.e.
\begin{align}\label{eq:LASSO}
\min_\theta \frac{1}{N}\sum_i (x_i\theta^\top-y_i)^2 + \lambda	|\theta|_1.
\end{align}
We have used the following datasets:
\begin{enumerate}[label=(\roman*)]
\item Household power consumption~\cite{electrdata}, which consists of $N=\num{2075259}$ data points.
We want to predict the \textit{Voltage} given \textit{active power, reactive power, intensity}.
We have raised to the tensor power of $5$\footnote{Raising to the tensor power of $\alpha$ means that we have added all the ``mixed'' products up to order $\alpha$: if we indicate with $x^i_m$, $i\in\{1,\ldots,n\}$, the $i$-th feature of the $m$-th point, in the case $\alpha=3$, we create all the new features of the form $x^i_m\times x^j_m$ and $x^i_m\times x^j_m\times x^h_m$,  $i,j,k\in\{1,\ldots,n\}$ for all the points $m\in\{1,\ldots N\}$.},
scaled the data, and applied PCA to reduce the number of features to \num{7}.
\item 3D Road Network~\cite{3ddata}, which consists of $N=\num{434874}$ data points.
We want to predict the \textit{Altitude}, given \textit{Longitude} and {Latitude}.
We have raised to the tensor power of $5$, scaled the data, and applied PCA to reduce the number of features to \num{7}.
\item NYC Taxi Trip Duration~\cite{nydata}, which consists of $N=\num{1458644}$ data points.
We want to predict the \textit{trip duration}, given \textit{pickup time/longitude/latitude} and \textit{dropoff longitude/latitude}. We consider only the time of the feature \textit{pickup\_datetime}, without the date.
We have raised to the tensor power of $3$, scaled the data, and applied PCA to reduce the number of features to \num{8}.
In this case we have considered as outliers the points such that $y_i>$\num{10000} -- this amounts to \num{2123} points (\num{0.14}\%).
\end{enumerate}
In all datasets the variance reduction by PCA is greater than \num{99.9}\%, which results from eliminating the symmetries introduced via the tensor power.
Throughout we have chosen $\lambda=0.01$ for the Lasso regularisation.

We have implemented the BCD with and without the Carath\'eodory sampling procedure with Gauss-Southwell rule (\textit{CaBCD GS}, \textit{BCD GS}), 
with a momentum strategy and the GS rule (\textit{CaBCD mom GS}, \textit{BCD mom GS}),
and with the Random rule (\textit{CaBCD mom random}, \textit{BCD mom random}).
For the momentum strategy we have chosen the momentum parameter $\beta = 0.9$.
As benchmarks we used ADAM~\cite{Kingma2014} and SAG~\cite{Roux2012} with standard mini-batches with size of $256$.
The learning rate for the CaBCD
Algorithms and ADAM is \num{1e-3}, as suggested in~\cite{Kingma2014};
we selected it\_max\_Ca $=1/\gamma/10=100$.
SAG %
was more sensitive to the step size and we decreased it to \num{1e-6}
to preserve the convergence.
\subsection{Discussion of results}
The results are summarized in Figure \ref{fig:results_expCaBCD}.
Overall CaBCD strongly outperforms the other methods and within the CaBCD variants the ones that use moments do better. 
\begin{figure}[hbt!]
\centering
        \includegraphics[height=2.8cm,width=0.32\textwidth]{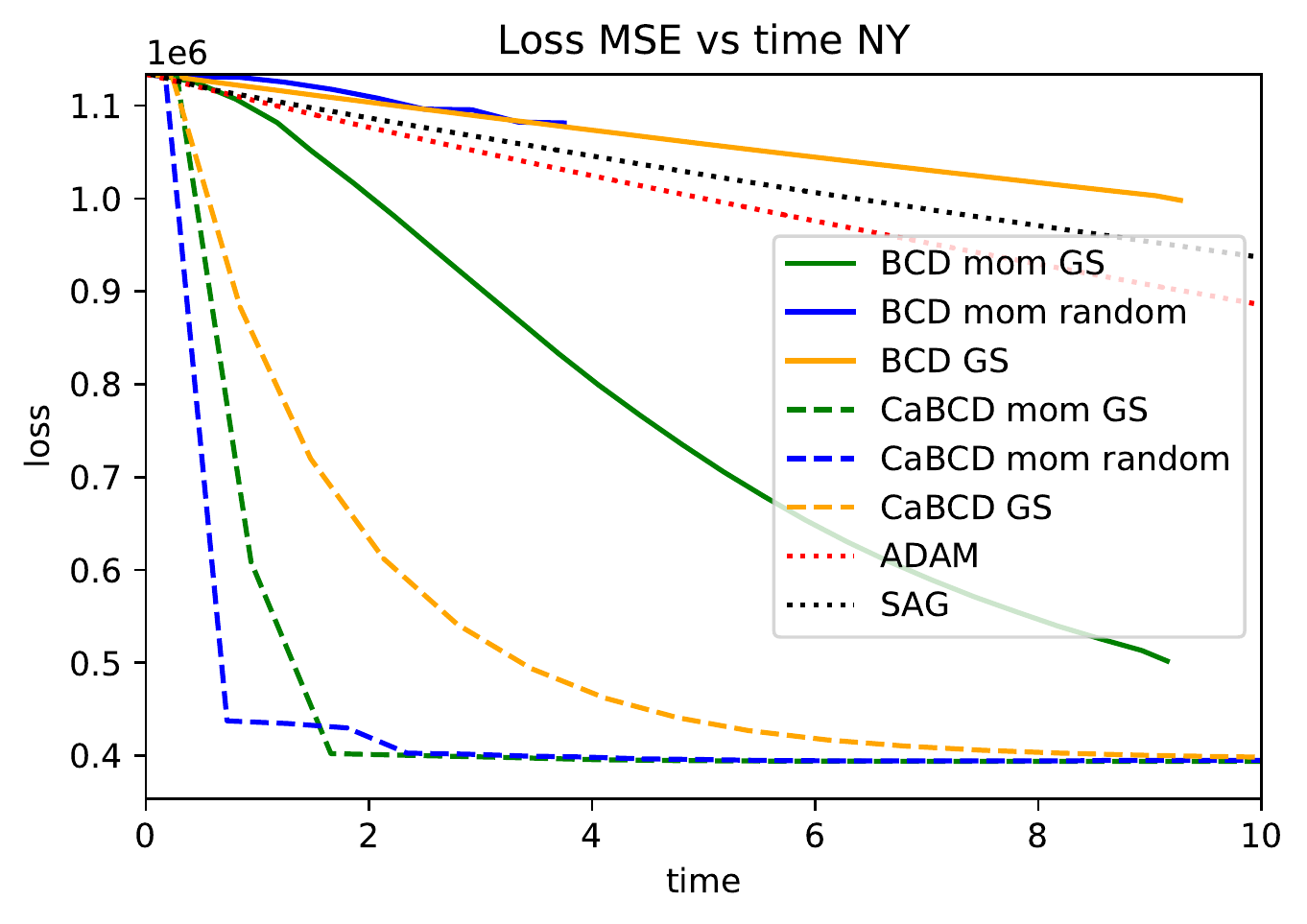}
        \includegraphics[height=2.8cm,width=0.32\textwidth]{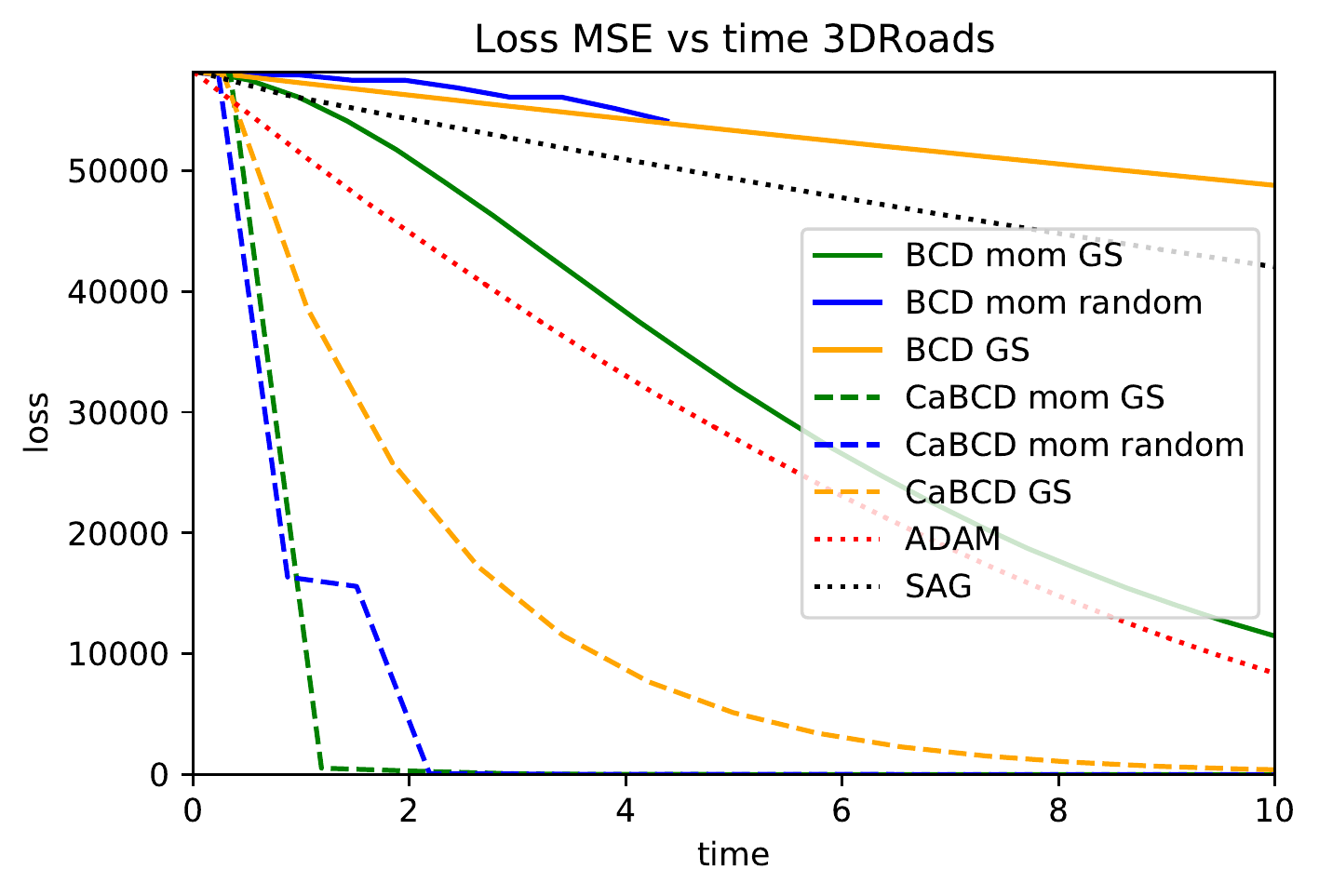}
        \includegraphics[height=2.8cm,width=0.33\textwidth]{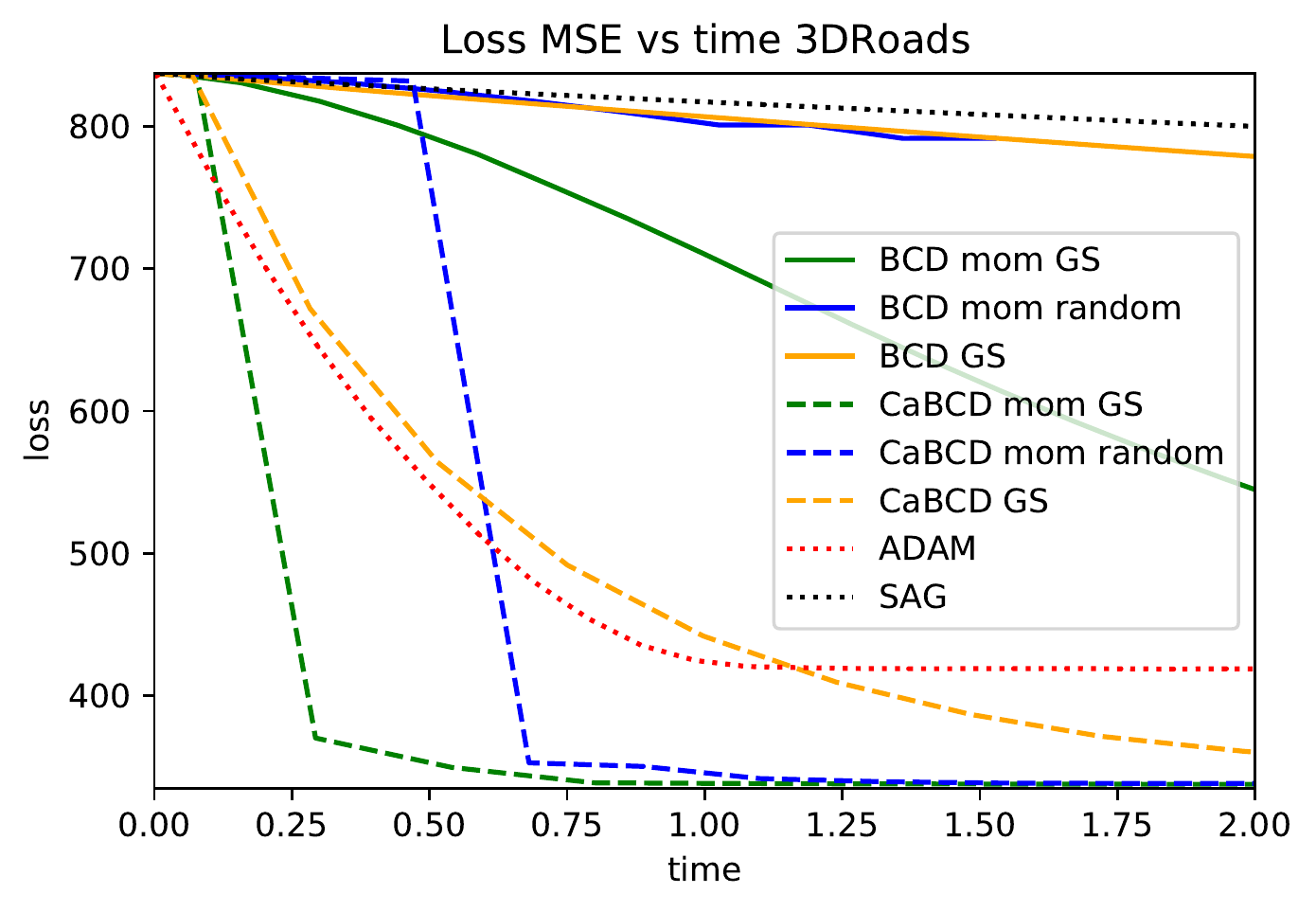}

        \includegraphics[height=2.8cm,width=0.32\textwidth]{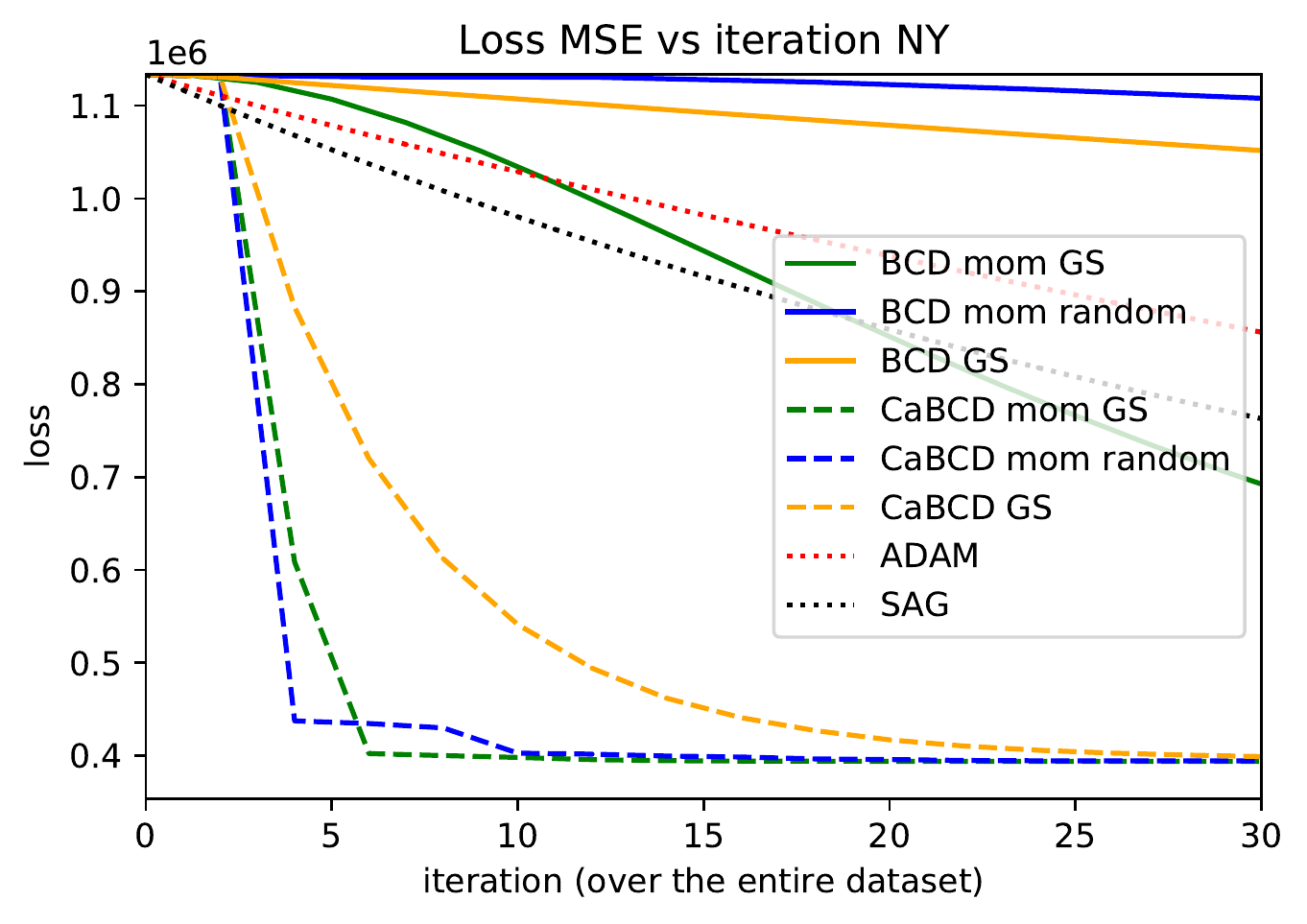}
        \includegraphics[height=2.8cm,width=0.32\textwidth]{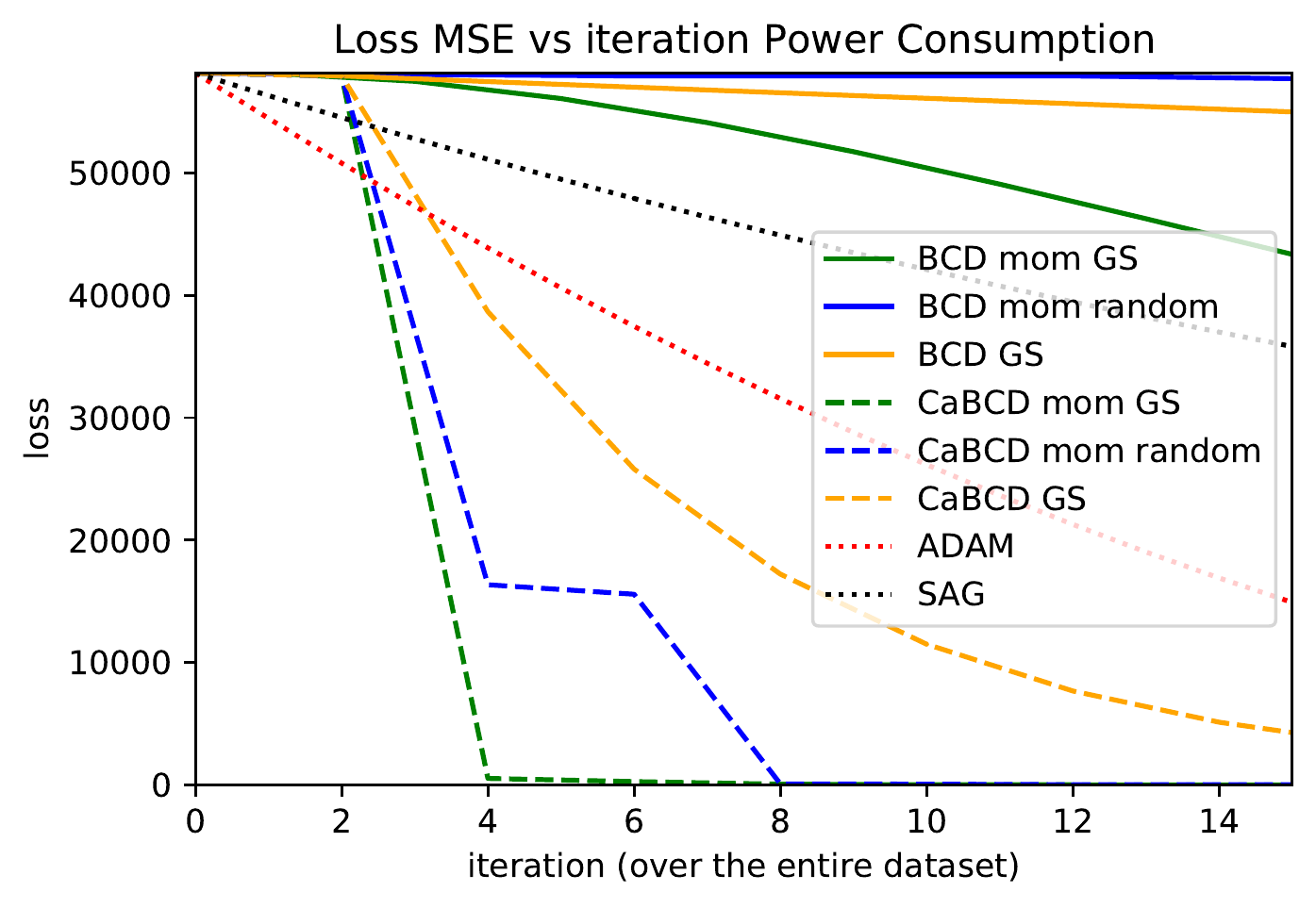}
        \includegraphics[height=2.8cm,width=0.32\textwidth]{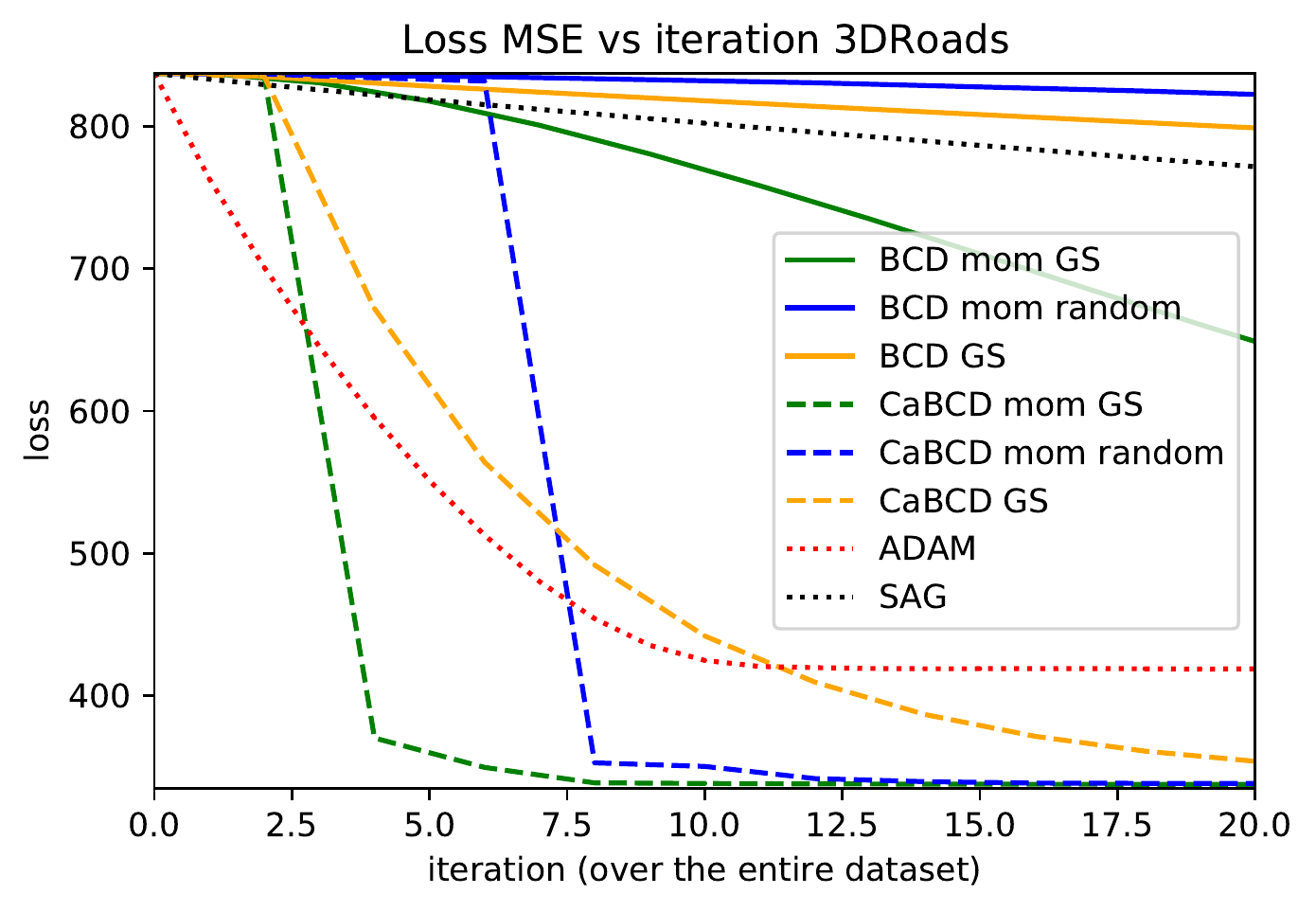}
    \caption{Running times and iterations of the different Algorithms.
    For \textit{CaBCD mom GS},  \textit{BCD mom GS} and \textit{CaBCD mom random}, \textit{BCD mom random}
    the directions have been computed using a standard momentum strategy, and chosen
    respectively by the GS rule and by the Random rule.
    For \textit{CaBCD GS}, \textit{BCD GS}
    the directions have been computed using the standard GD method, and chosen by the GS rule. }%
    \label{fig:results_expCaBCD}
    \end{figure}
    Some further observations are that, firstly, the size of the blocks $s$ has been fixed to two.
    The reason is that experimentally we have observed that if the block's size is between $2$ and $5$ the reduced measure is used for longer, i.e.~the algorithm does more steps with the reduced measure, thus decreasing the runtime.
    Secondly, in the case of CaBCD algorithms we count $1$ iteration when a full gradient has been computed, while we count $\frac{\text{number of points in the reduced measure}}{N}$ for any iteration done with the reduced measure (if the size of the block is $s$, the reduced measure has support on $s+1$ points, see Theorem \ref{th:cath}).
    An analogous reasoning is used to count the iterations of SAG and ADAM.
    Third, the CaBCD algorithms are for the first two iterations are ``slower''.
    This is due to the fact that we compute $\mathcal{H}$, i.e.~the approximation of the second derivative.
    Finally, using the GS rule, the parallelisation of the code has often no effect because the directions to optimise belong to only one block.

\subsection{Let's make Carath\'eodory Block Coordinate Gradient Descent go fast}
The central question of BCD is the choice of the update rule.
In the previous section we used the arguably simples ones, randomized and Gauss--Southwell, for CaBCD. 
However, more sophisticated update rules are possible which in turn could lead to a further performance improvement. 
{To understand this better, we revisit in this section the study of different BCD rules of~\cite{Nutini2017} in the context of our CaBCD.
To do so we follow \cite{Nutini2017} and focus on a least-squares problem 
\begin{align}\label{eq:ls}
\min_\theta \sum_i (x_i\theta^\top-y_i)^2.
\end{align}
We use \cite[Dataset A]{Nutini2017} %
with $N=\num{1000000}$ and $n=\num{500}$.
The data are generated following the same procedure explained in \cite[Appendix F.1]{Nutini2017}.
The $x_i$ values are sampled from a standard normal random variable, then
1 is added to induce a dependency between columns and each
column is multiplied by a sample from a standard normal random variable multiplied by ten, to induce different Lipschitz constants across the
coordinates. Finally, each entry is kept non-zero with probability $10 \log(m)/m$.
$y_i = x_i\cdot \theta^\times + e_i$, where the $e_i$ are drawn from a standard normal random variable.
90\% of $\theta^\times $ is set to zero and the remaining values are sampled from a standard normal random variable.}
\begin{figure}[t!]
\centering
\includegraphics[width=0.335\textwidth, clip=true, trim = 0 125mm 0 0 ]{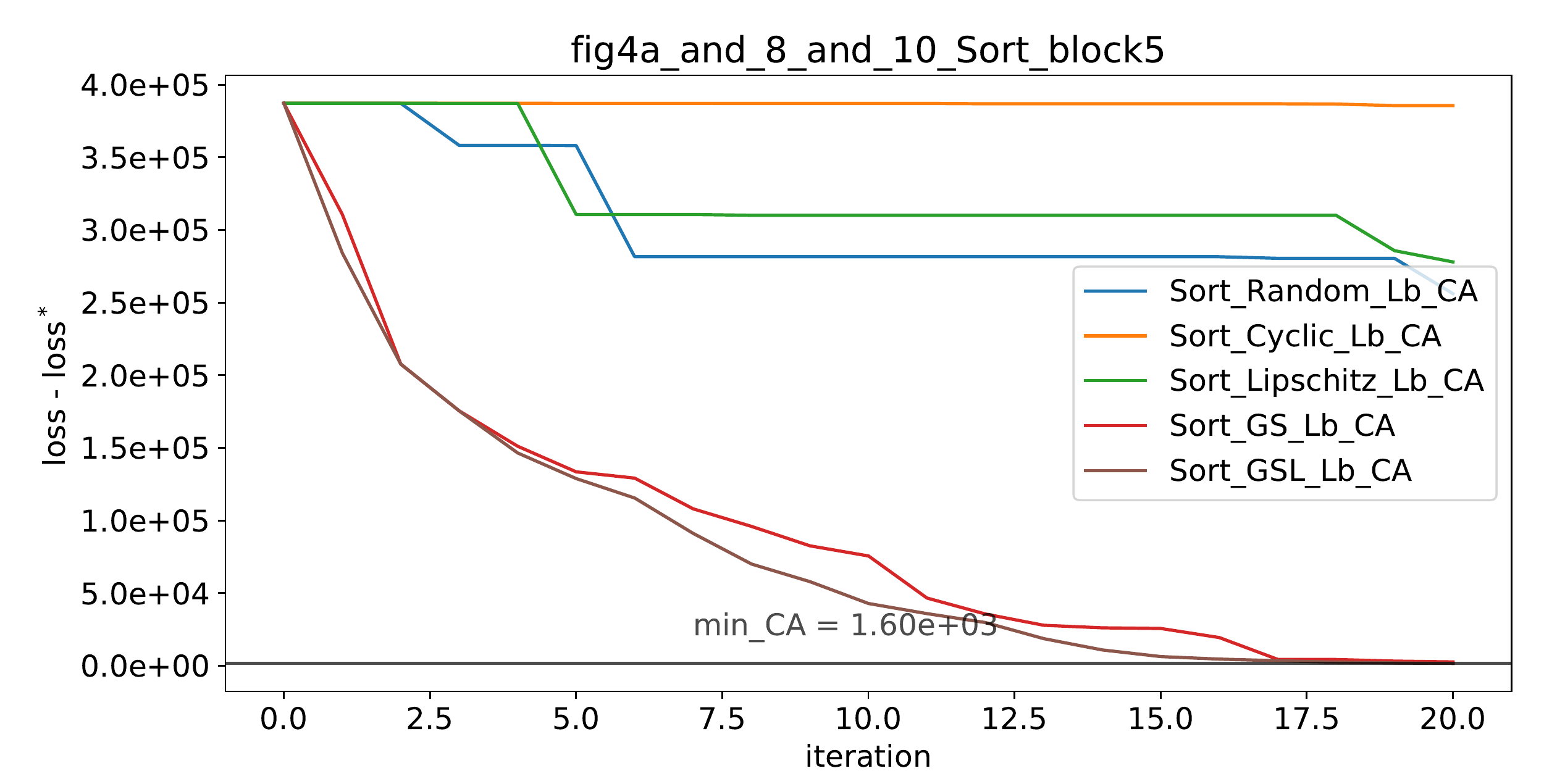}
\includegraphics[width=0.32\textwidth, clip=true, trim = 12mm 125mm 0 0 ]{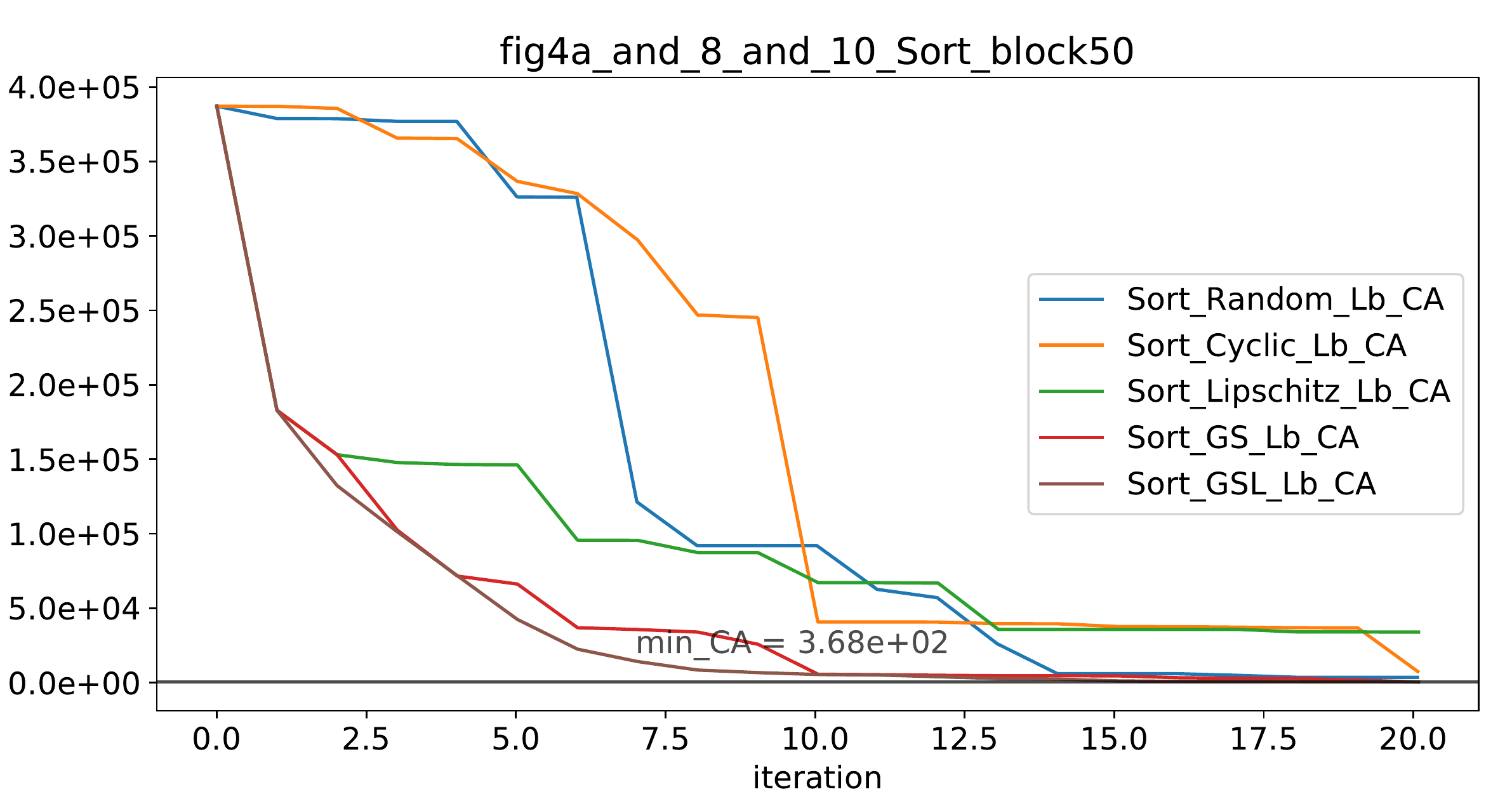}
\includegraphics[width=0.32\textwidth, clip=true, trim = 12mm 125mm 0 0]{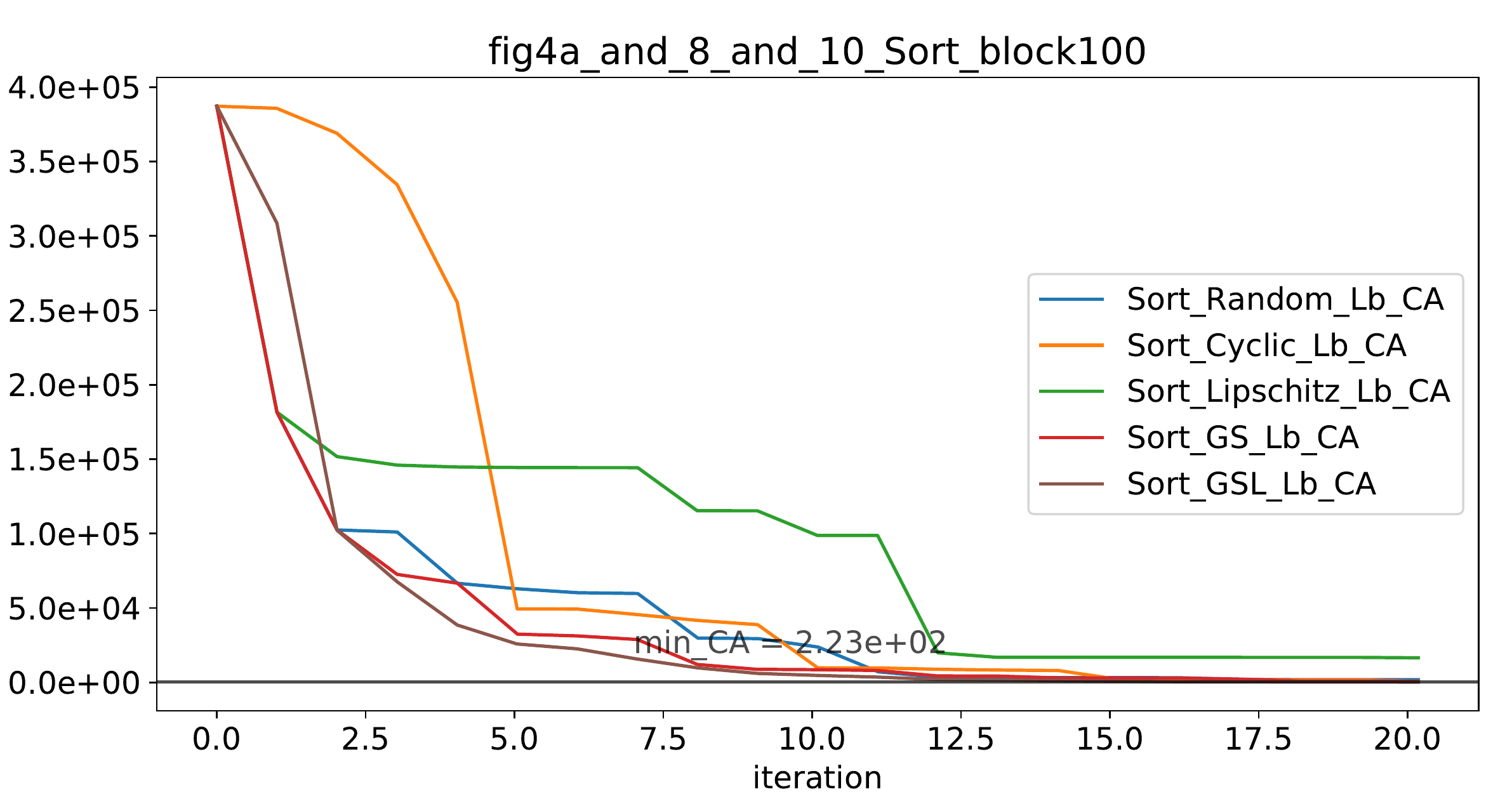}

\includegraphics[width=0.335\textwidth, clip=true, trim = 0 125mm 0 0 ]{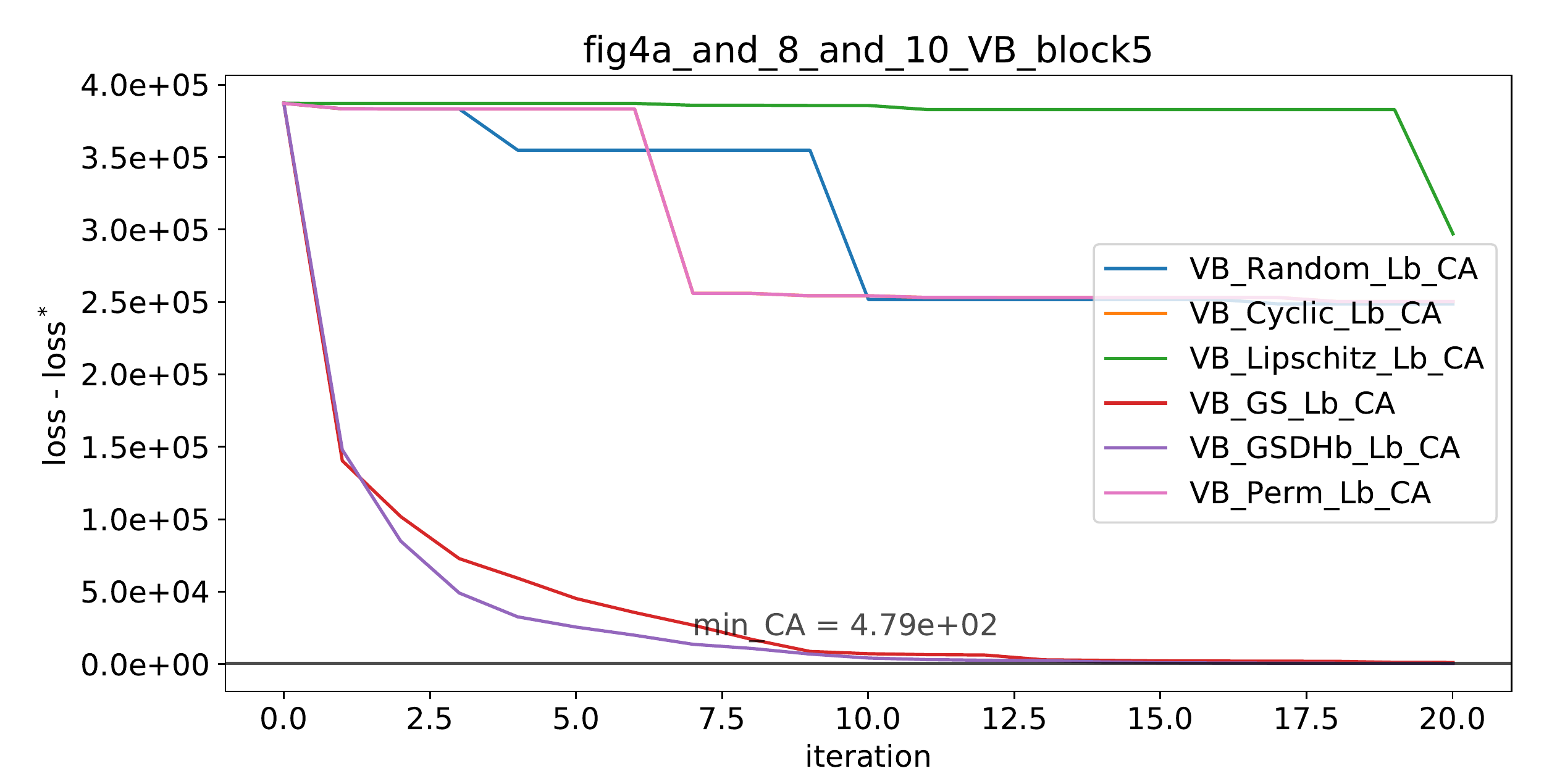}
\includegraphics[width=0.32\textwidth, clip=true, trim = 12mm 125mm 0 0]{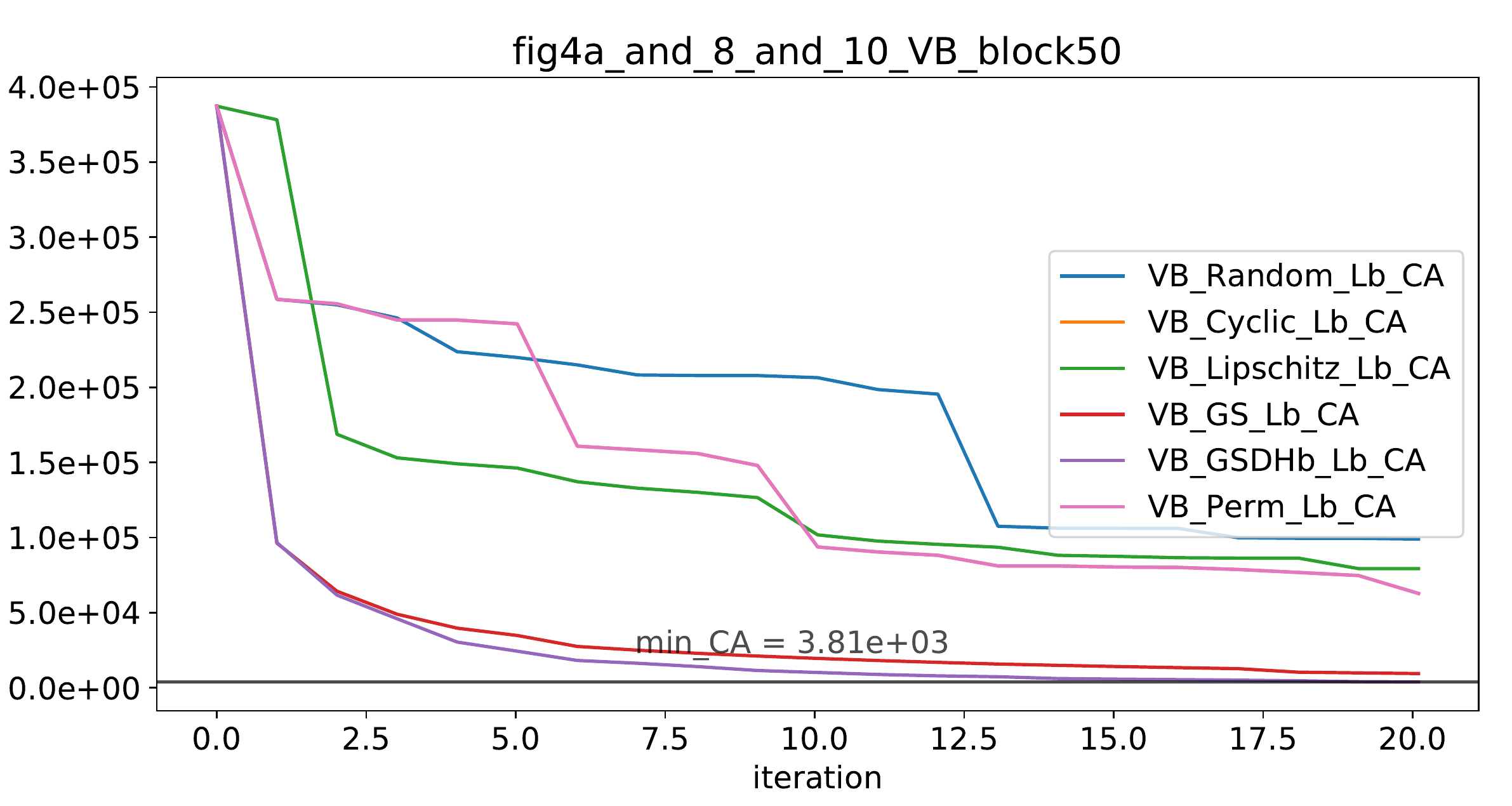}
\includegraphics[width=0.32\textwidth, clip=true, trim = 12mm 125mm 0 0]{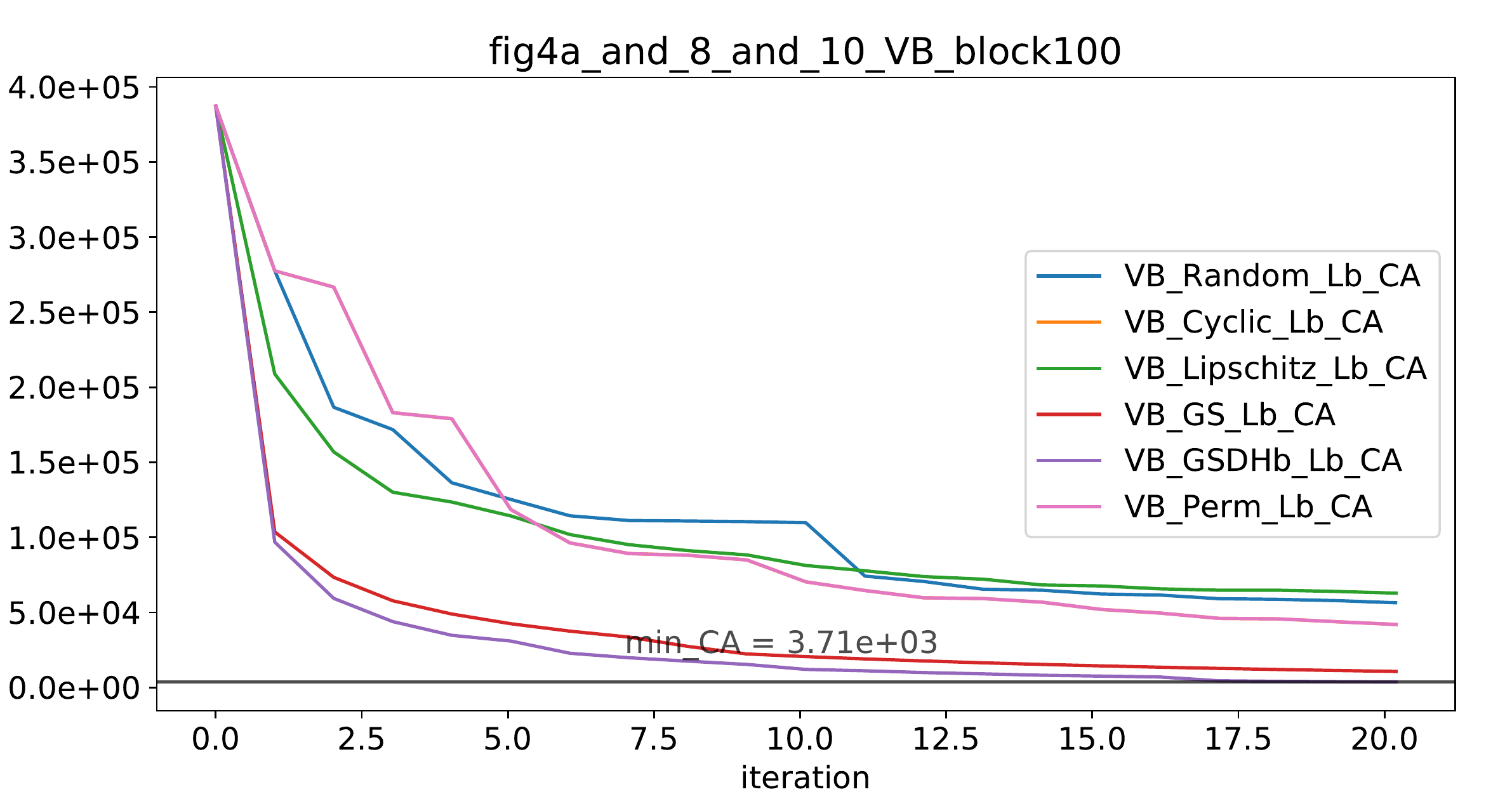}
\includegraphics[width=0.335\textwidth, clip=true, trim = 0 125mm 0 0 ]{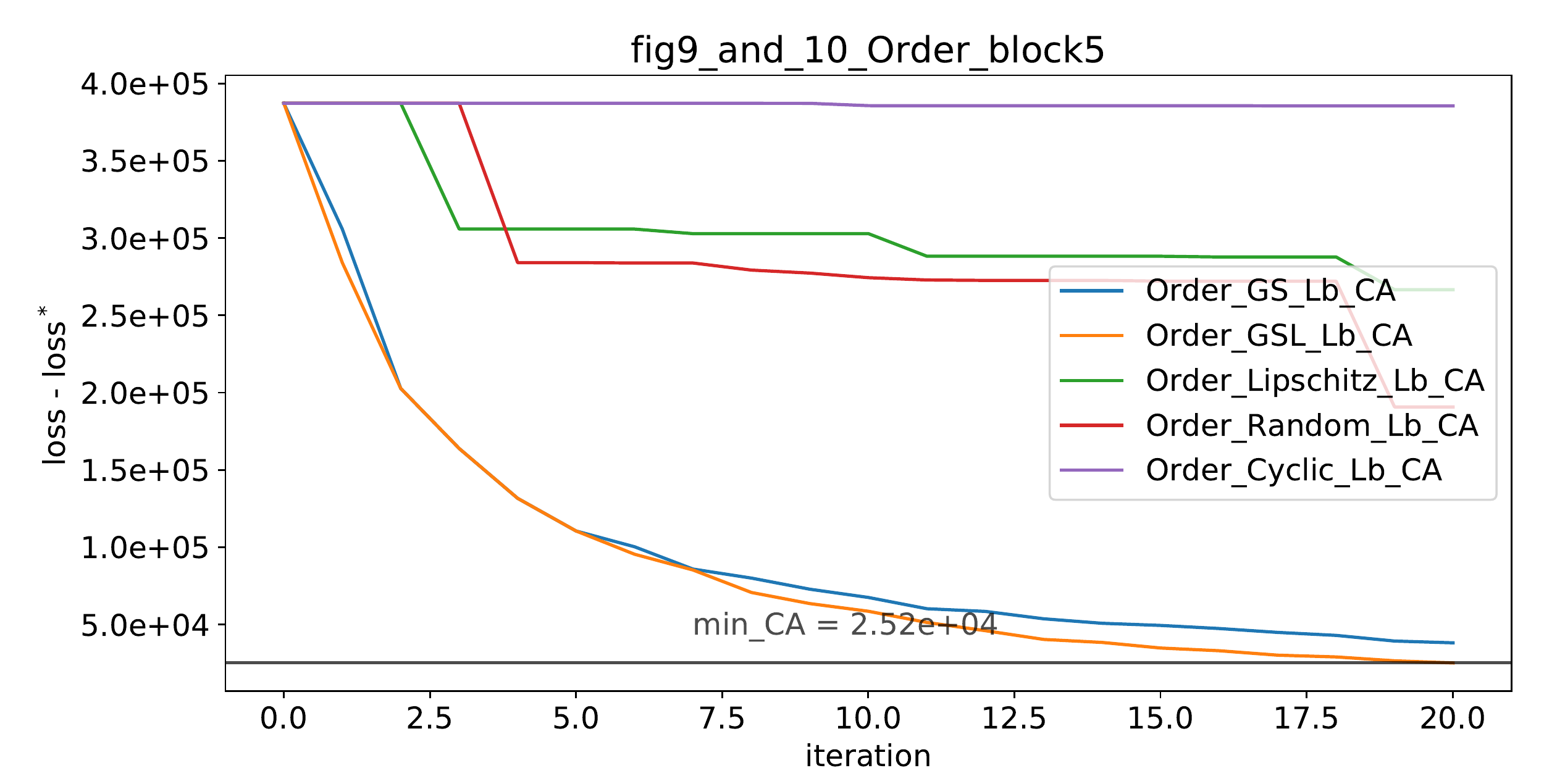}
\includegraphics[width=0.32\textwidth, clip=true, trim = 12mm 125mm 0 0]{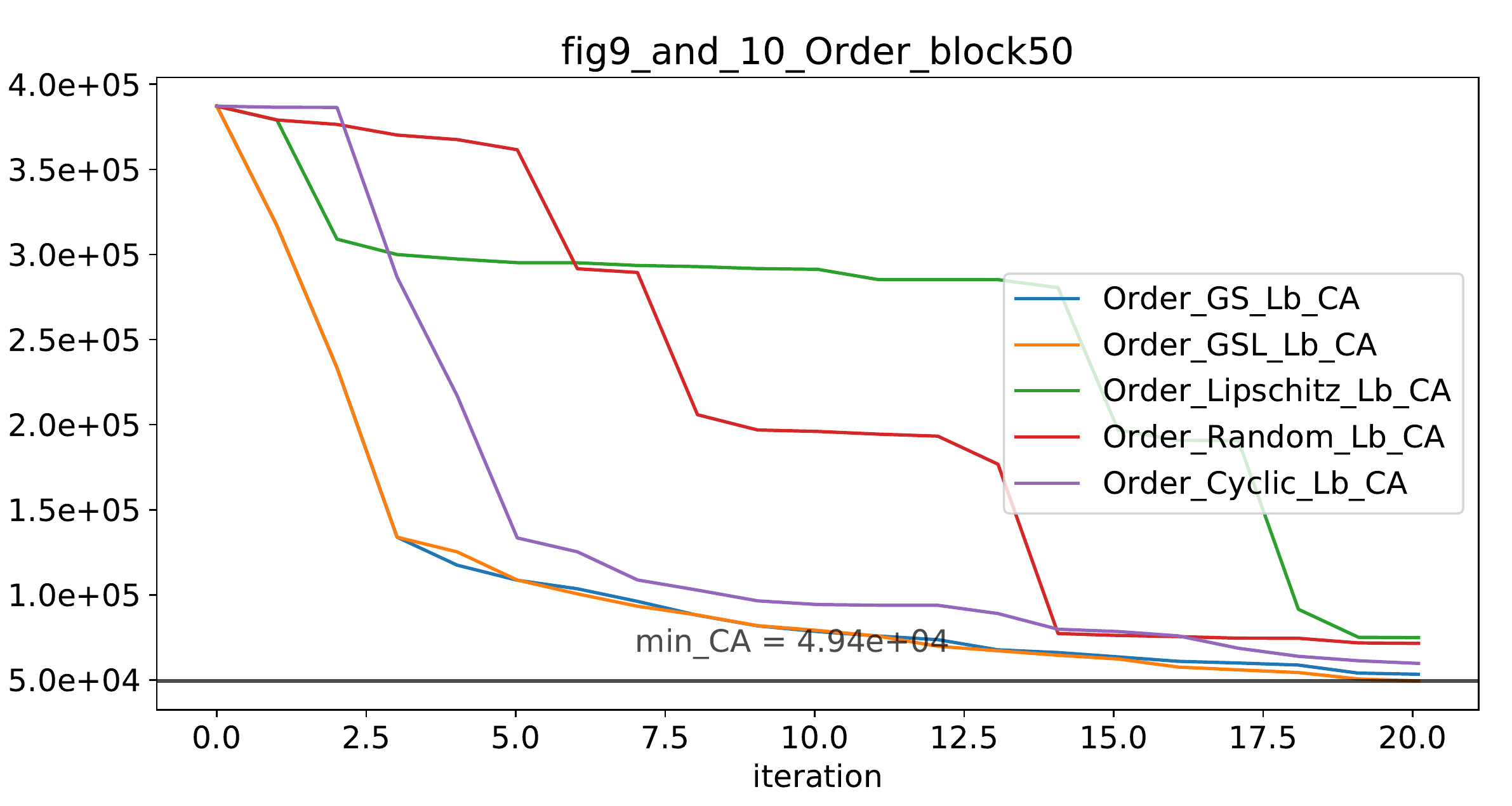}
\includegraphics[width=0.32\textwidth, clip=true, trim = 12mm 125mm 0 0]{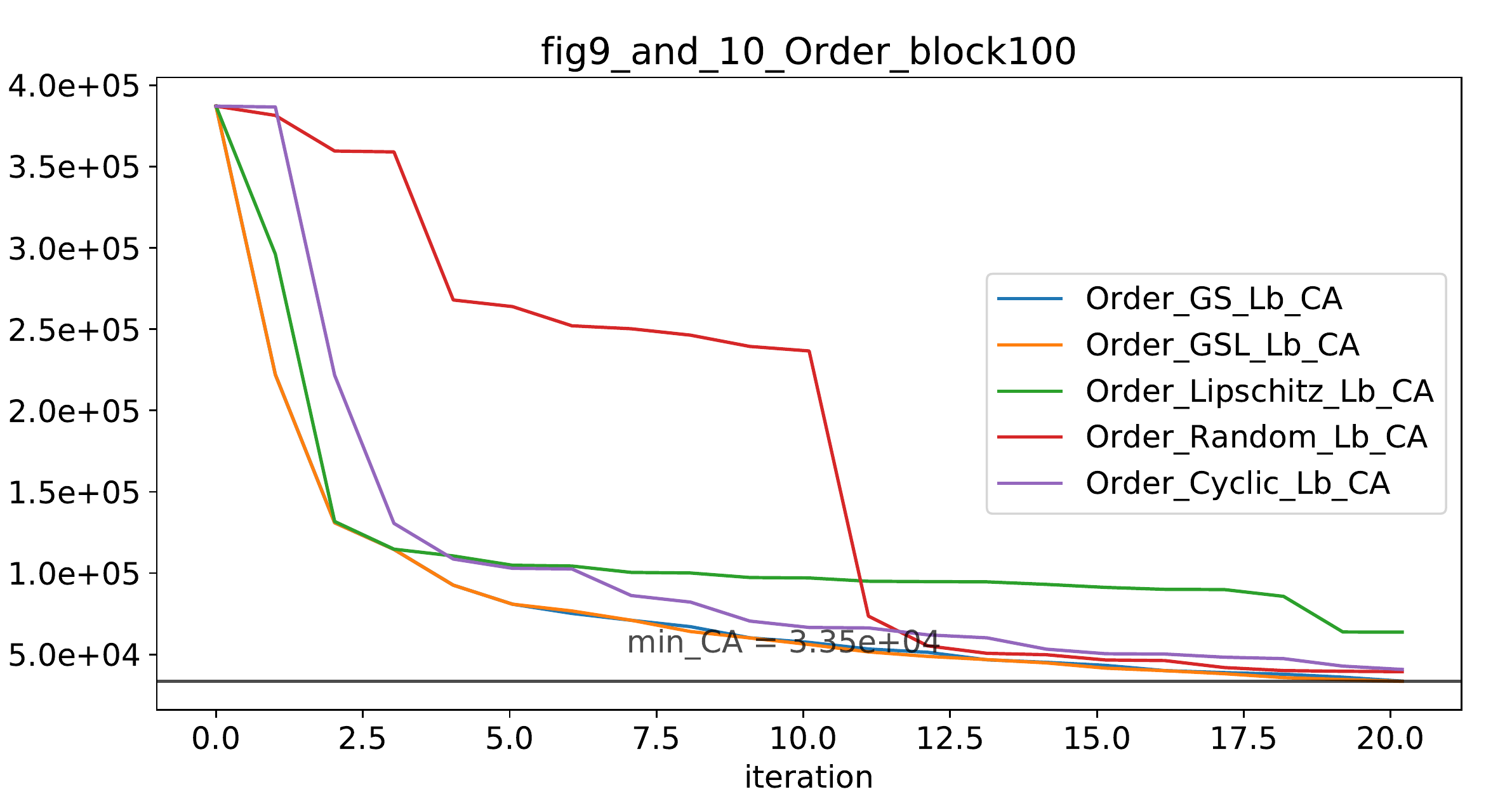}
\includegraphics[width=0.335\textwidth, clip=true, trim = 0 125mm 0 0 ]{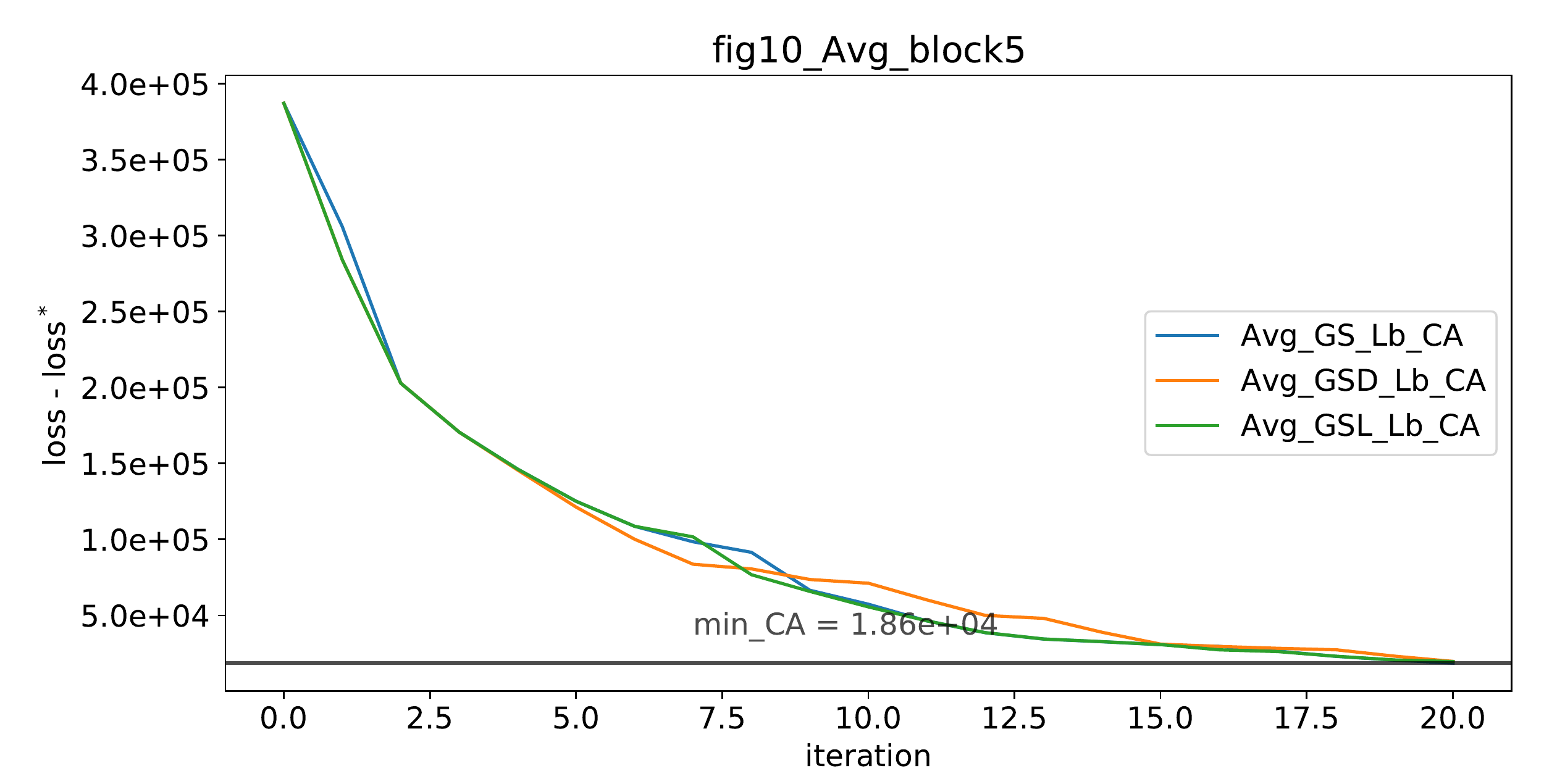}
\includegraphics[width=0.32\textwidth, clip=true, trim = 12mm 125mm 0 0]{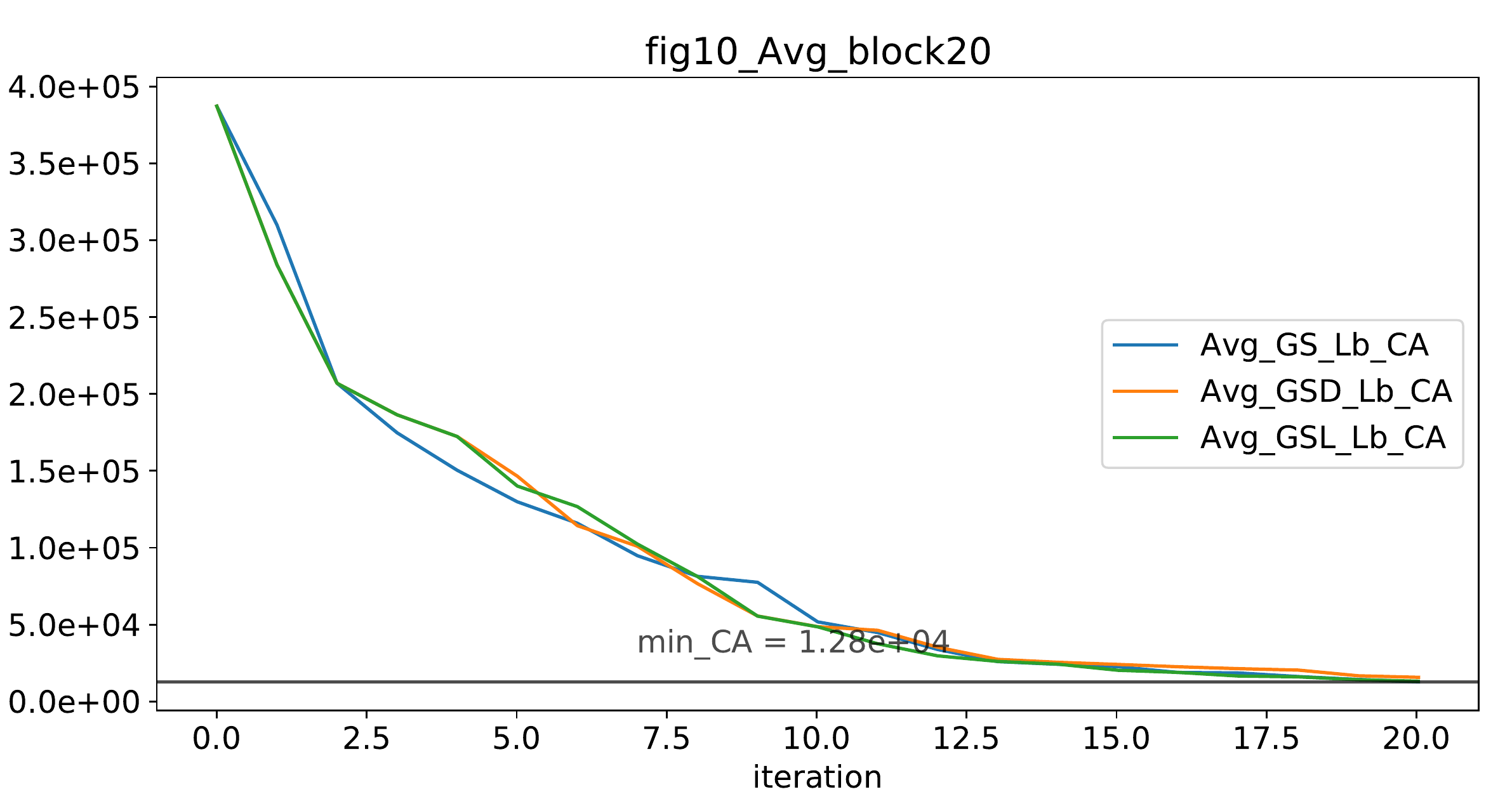}
\includegraphics[width=0.32\textwidth, clip=true, trim = 12mm 125mm 0 0]{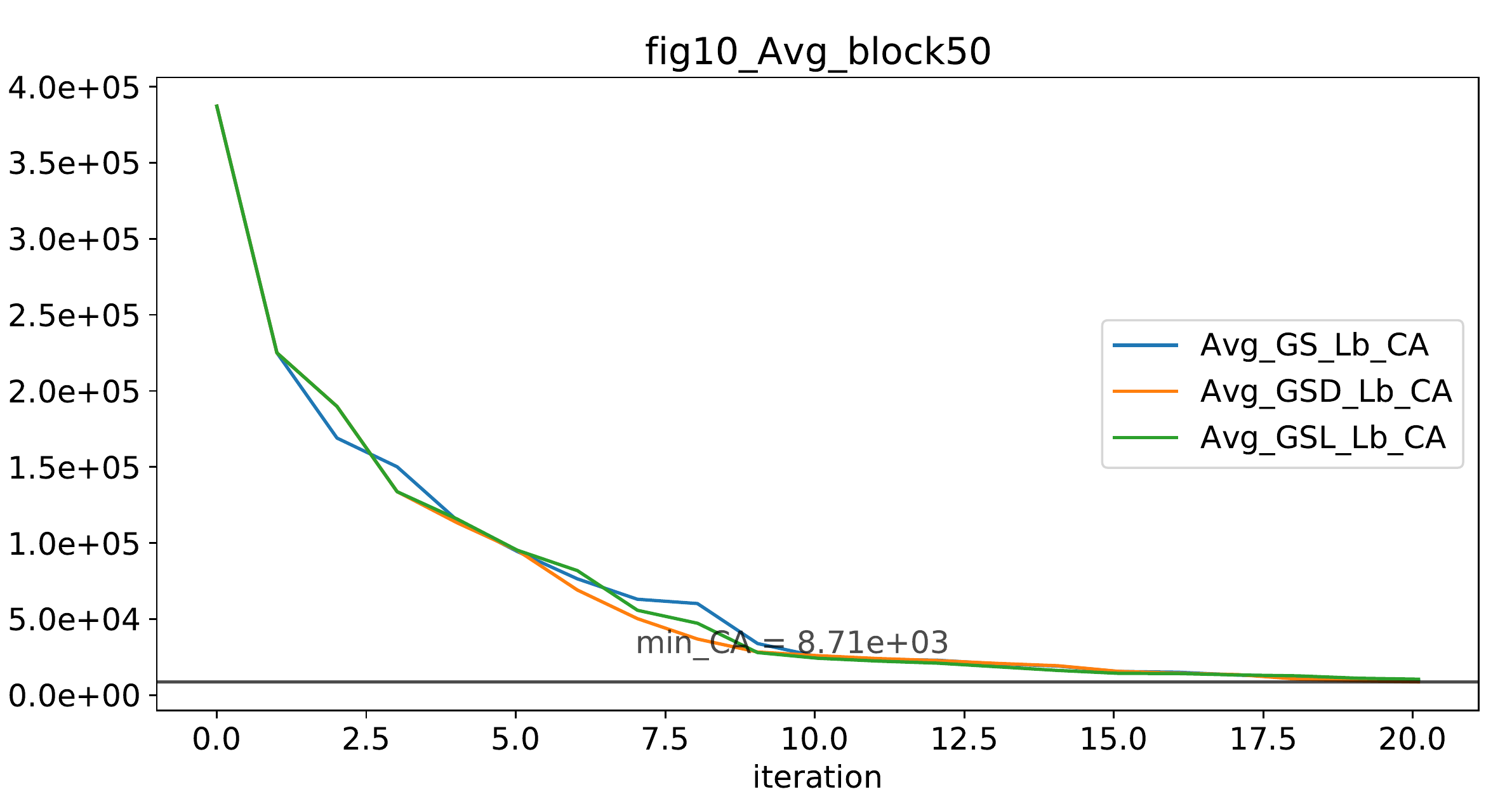}

\caption{{CaBCD applied with different block sizes, rules used in 
\cite[Figure 4, 8, 10]{Nutini2017} (top-two lines), 
\cite[Figure 9]{Nutini2017} (third line) and 
\cite[Figure 10]{Nutini2017}.}
}\label{fig:nutinifig4,8,9,10}
\end{figure}
\subsubsection{BCD update rules}
{ %
The rules presented in~\cite{Nutini2017} can be represented as
\[
\theta_{j+1} = \theta_{j} + \langle\Gamma,\sum_i \nabla L_{\theta_j} (x_i,y_i)\rangle,
\]
where $\Gamma$ can be a function of the Hessian of $L$, of Lipschitz bounds related to $L$, or however it can depend on $L_{\theta_j} (x_i,y_i)$ non linearly in the data, e.g. the inverse of the Hessian.
Due to the non-linearity, we compute the reduced measure for $\sum_i \nabla L_{\theta_j} (x_i,y_i)$ and consider $\Gamma$ as an independent factor.
{In general, Lipschitz bounds are difficult to find, whilst precise Hessian information is expensive computationally unless a closed formula is available, which is the case only for a small portion of models}.\\
In~\cite{Nutini2017} the step-size of the BCD is determined as a function of the computed Lipschitz bounds. While using the recombined measure we use a factor $\gamma=\num{1e-2}$, i.e.
\[
\hat\theta_{j+1} = \hat\theta_{j} + \gamma \times \langle\Gamma,\sum_i \nabla L_{\hat\theta_j} (\hat x_i,\hat y_i)\rangle.
\]
In place of $\EE[\nabla L_{\theta_j} (X,Y)] = \frac{1}{N}\sum_i \nabla L_{\theta_j} (x_i,y_i)$, in~\cite{Nutini2017} $\sum_i \nabla L_{\theta_j} (x_i,y_i)$ is used, which results in higher loss values.\\
Compared to the previous experiments, we want to underline that some of the rules used in~\cite{Nutini2017} compute precisely the Lipschitz Constants of the different blocks. %
Indeed, for least-squares problems Lipschitz constants can be written expliclity as a function of $(x_i,y_i)$ and $\theta_j$, see e.g.~\cite[Appendix B]{Nutini2017}.\\
Dataset A used in~\cite{Nutini2017} is synthetic and sparse.
While the rules to select the directions of~\cite{Nutini2017} prefer sparse matrices, we did not optimize the algorithms' code to find the reduced measures to efficiently deal with sparse datasets.
Nevertheless, we can imagine improvements given a significant presence of matrices multiplications in the implementations.}

\subsection{A list of rules.}
{We briefly introduce the rules below, for an exhaustive description we refer to~\cite{Nutini2017}.
We structure the experiments and plots as follows:
any rule is represented by the following string format
\begin{center}
\textit{``partition\_block-selection\_direction''}
\end{center}
with an additional suffix ``\textit{\_CA}'' to note that we have applied it with CaBCD.
The possible choices for \emph{partition}, \emph{block-selection}, \emph{direction} are\footnote{
Not all the combinations are possible, see \cite{Nutini2017} and the official repository \url{https://github.com/IssamLaradji/BlockCoordinateDescent} from more details. }
  \begin{align}
    \text{ partition} \in & \{\text{VB, Sort, Order, Avg}\}, \\
    \text{block-selection} \in &\{\text{Random, Cyclic, Lipschitz, Perm, GS, GSD, GSL,} \\
    										&\,\,\,\,\,\, \text{GSDHb, GSQ, IHT}\},\\
    \text{direction} \in & \{\text{Hb, Lb}\}.
  \end{align}
We give details on the choices below: 
VB stands for Variable Blocks which indicates that the partition of the directions can change at any iteration of the optimization procedure.
Sort fixes the partition from the beginning, organizing the blocks of directions according to their {Lipschitz values: the largest Lipschitz} values into the first block, and so on.
Order it fixes the partition from the beginning, subdividing the directions in order, e.g. if the block size is 2, the blocks will be $(1,2), (3,4), $ etc.
Avg fixes the partition alternating between adding large and small Lipschitz values.
Between the previous, VB is the only one which allows the partition to change between iterations. 
The ``\textit{block-selection}'' rules prescribe how blocks are selected given the partition of the directions and we refer to \cite{Nutini2017} for details.
The two choices of ``\textit{direction}'' are ``\textit{Lb}'' and ``\textit{Hb}''.
\textit{Lb} means that the direction for the update is {$G_{block} / L_{block}$; $Hb$ signifies that the direction is $ H_{block}^{-1}\cdot G_{block}$, where $L_{block}, G_{block},  H_{block}$ represent respectively the Lipschitz value, the Gradient and the Hessian of the chosen block.}}

The plots are named analogous{ly} to the plots in~\cite{Nutini2017} but additionally we include the values of the size of the blocks.
For the implementation of the block{s}' selection rules we have used the code provided by the authors of~\cite{Nutini2017}, freely available at
\url{https://github.com/IssamLaradji/BlockCoordinateDescent}. 

\begin{figure}[bth!]
\centering
\includegraphics[width=0.16\textwidth]{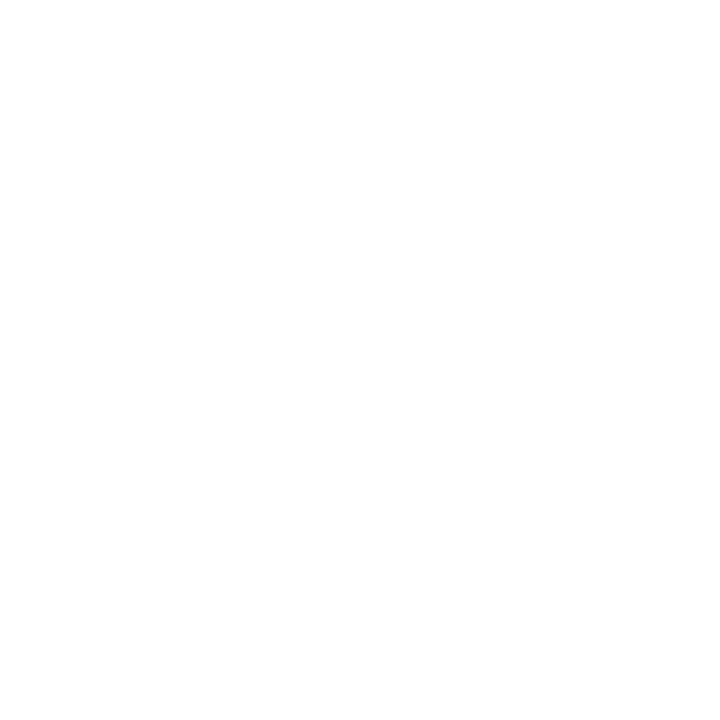}
\includegraphics[width=0.335\textwidth, clip=true, trim = 0 125mm 0 0 ]{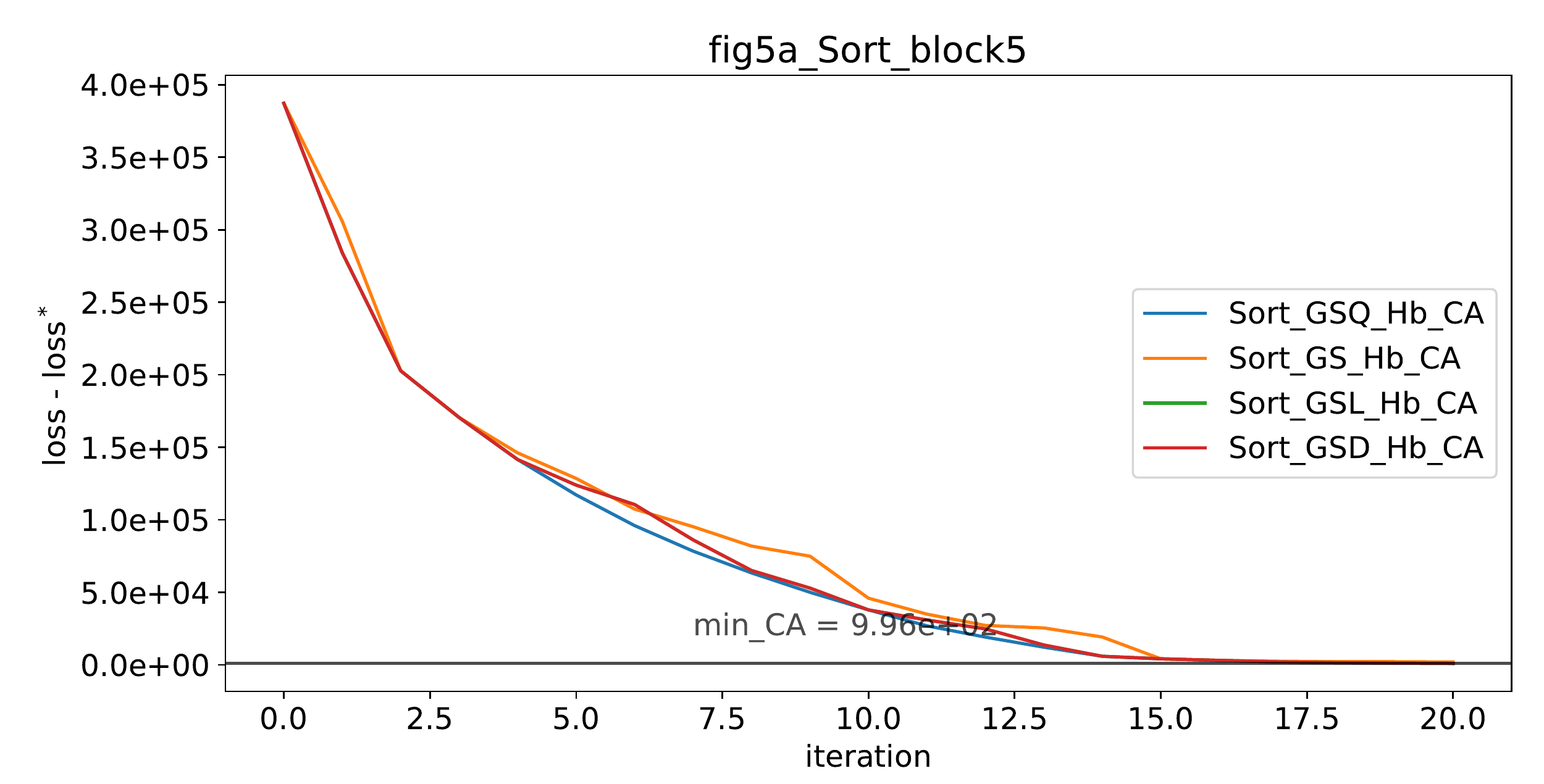}
\includegraphics[width=0.32\textwidth, clip=true, trim = 12mm 125mm 0 0]{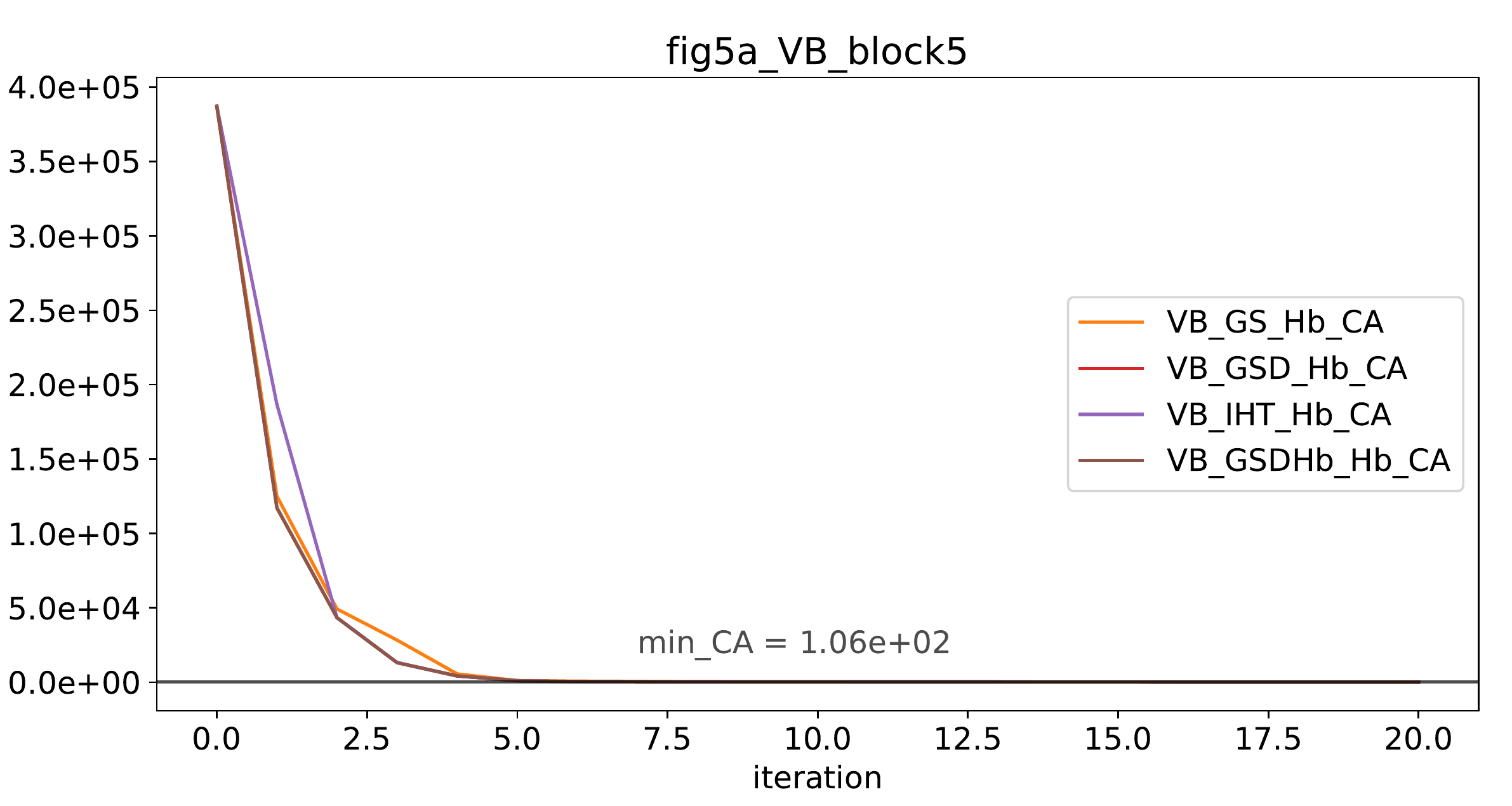}
\includegraphics[width=0.16\textwidth]{figure/figwhite}
\includegraphics[width=0.335\textwidth, clip=true, trim = 0 125mm 0 0 ]{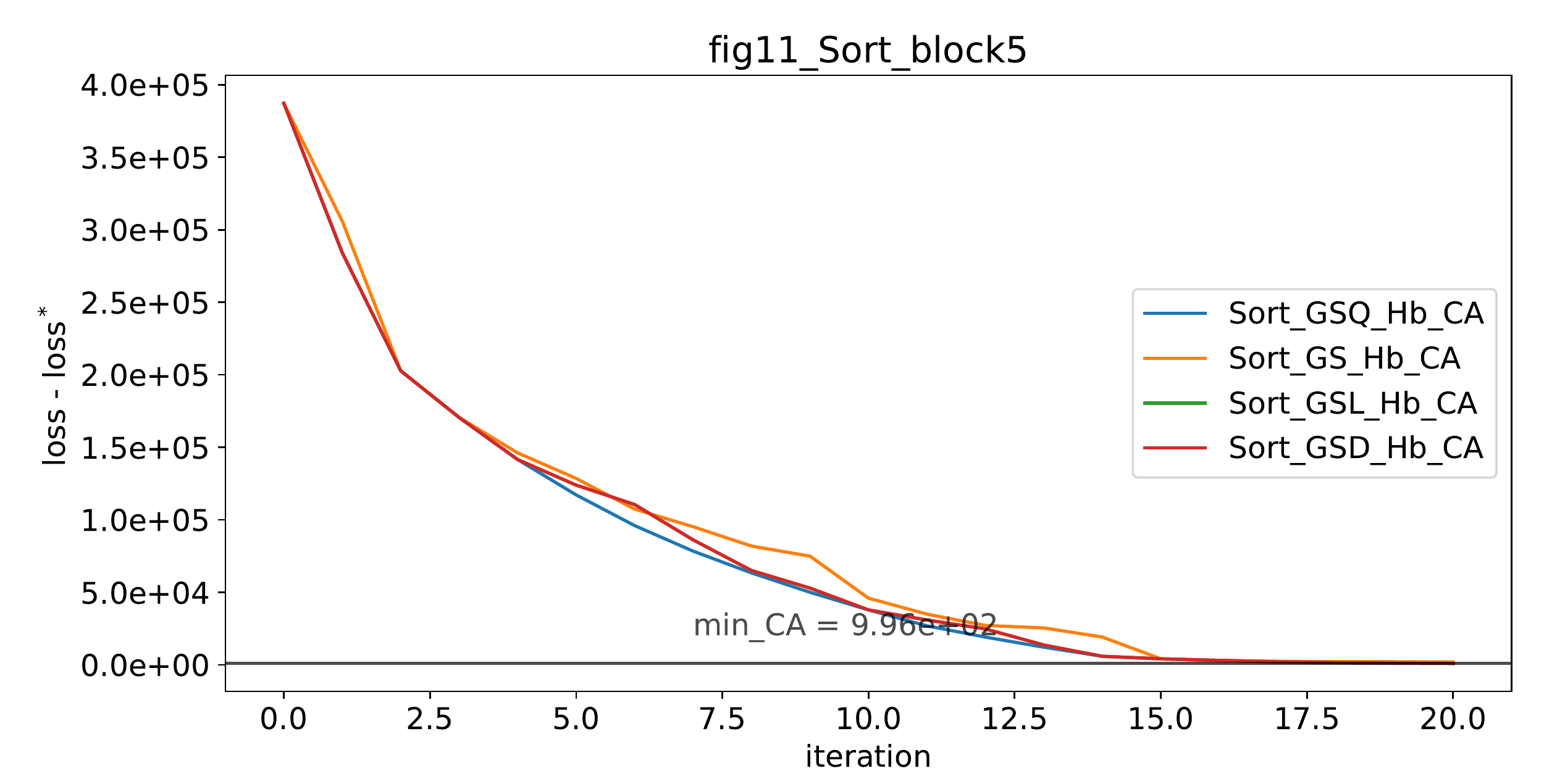}
\includegraphics[width=0.32\textwidth, clip=true, trim = 12mm 125mm 0 0]{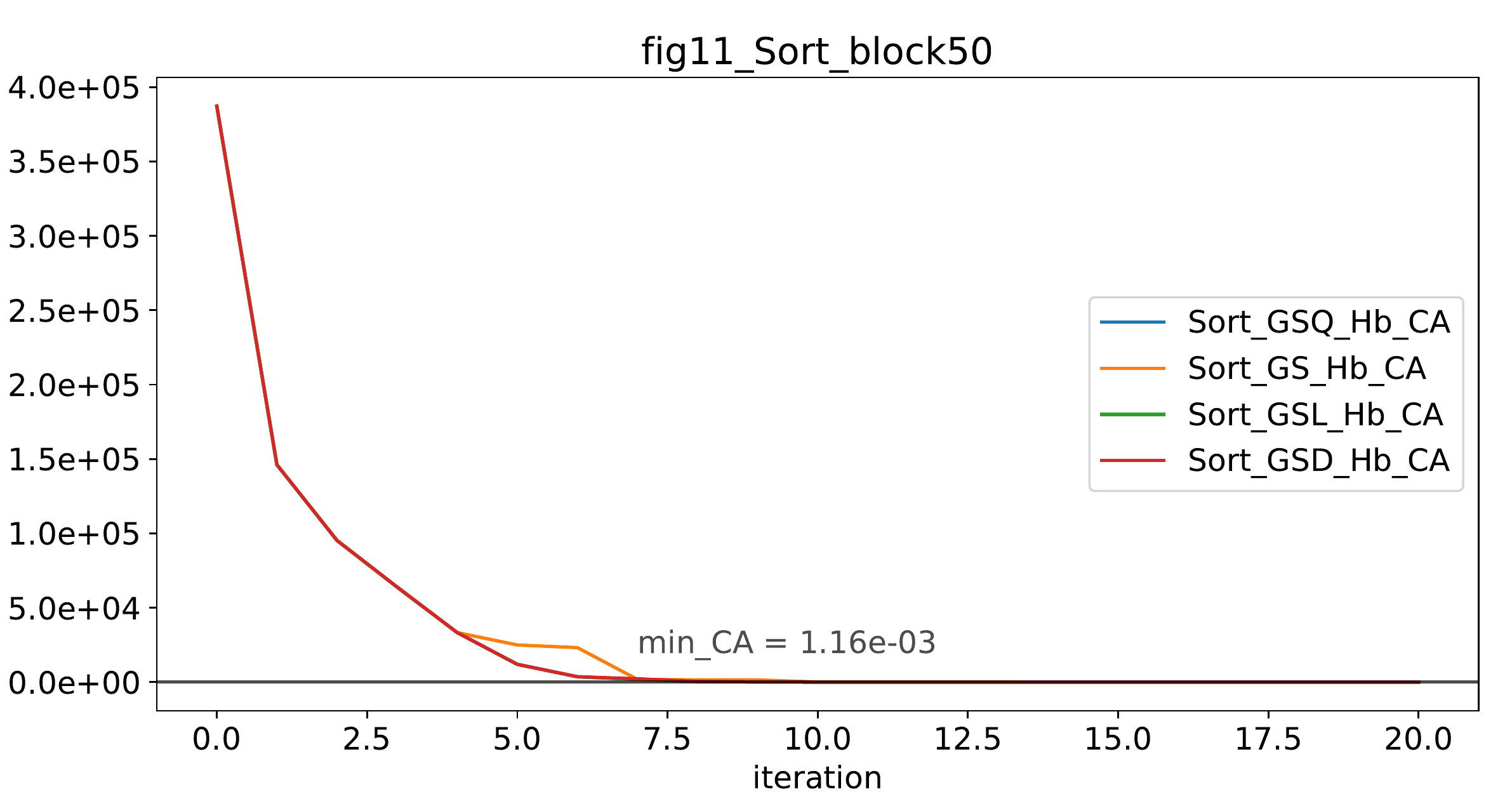}
\includegraphics[width=0.32\textwidth, clip=true, trim = 12mm 125mm 0 0]{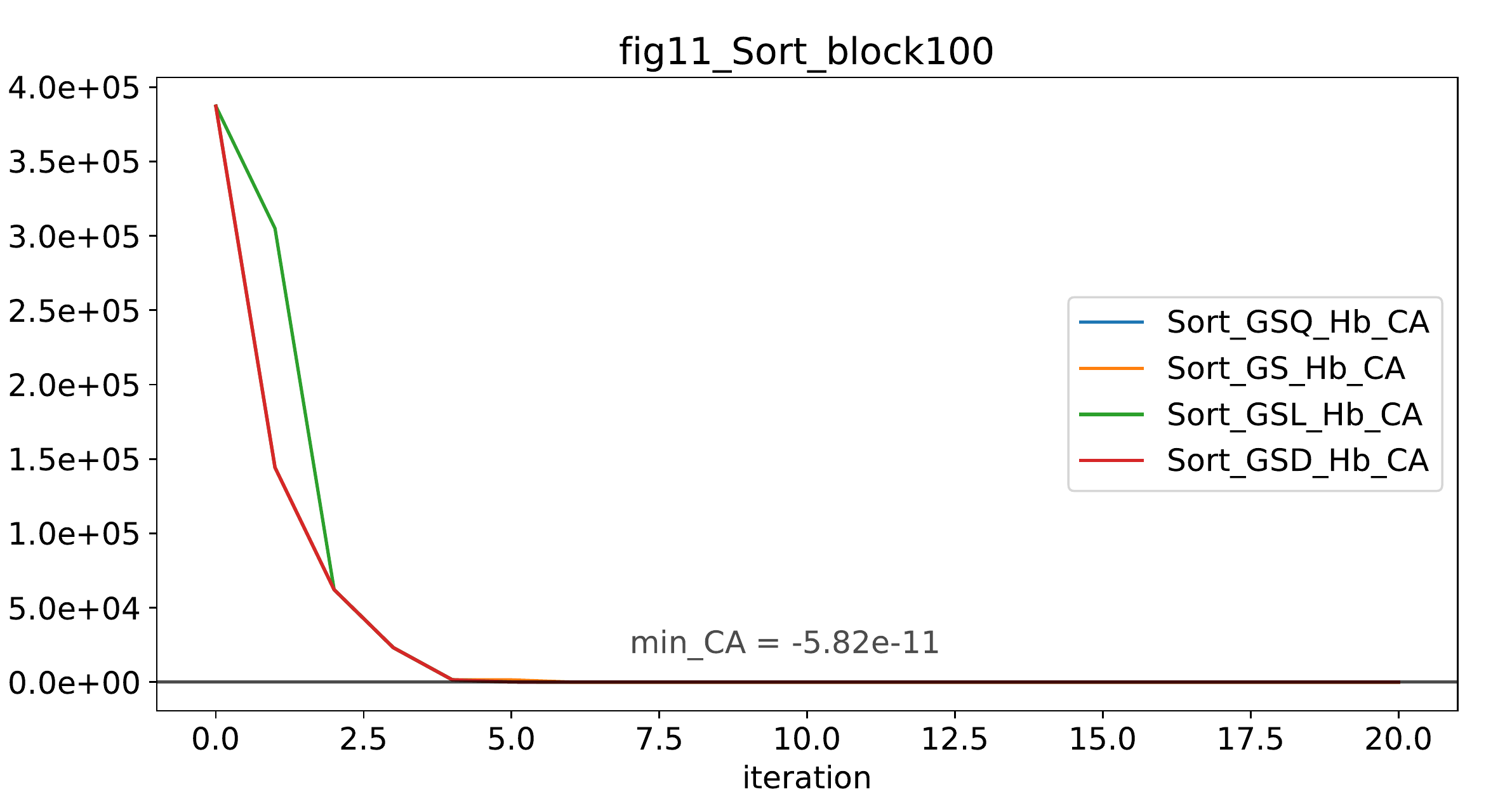}

\includegraphics[width=0.335\textwidth, clip=true, trim = 0 125mm 0 0 ]{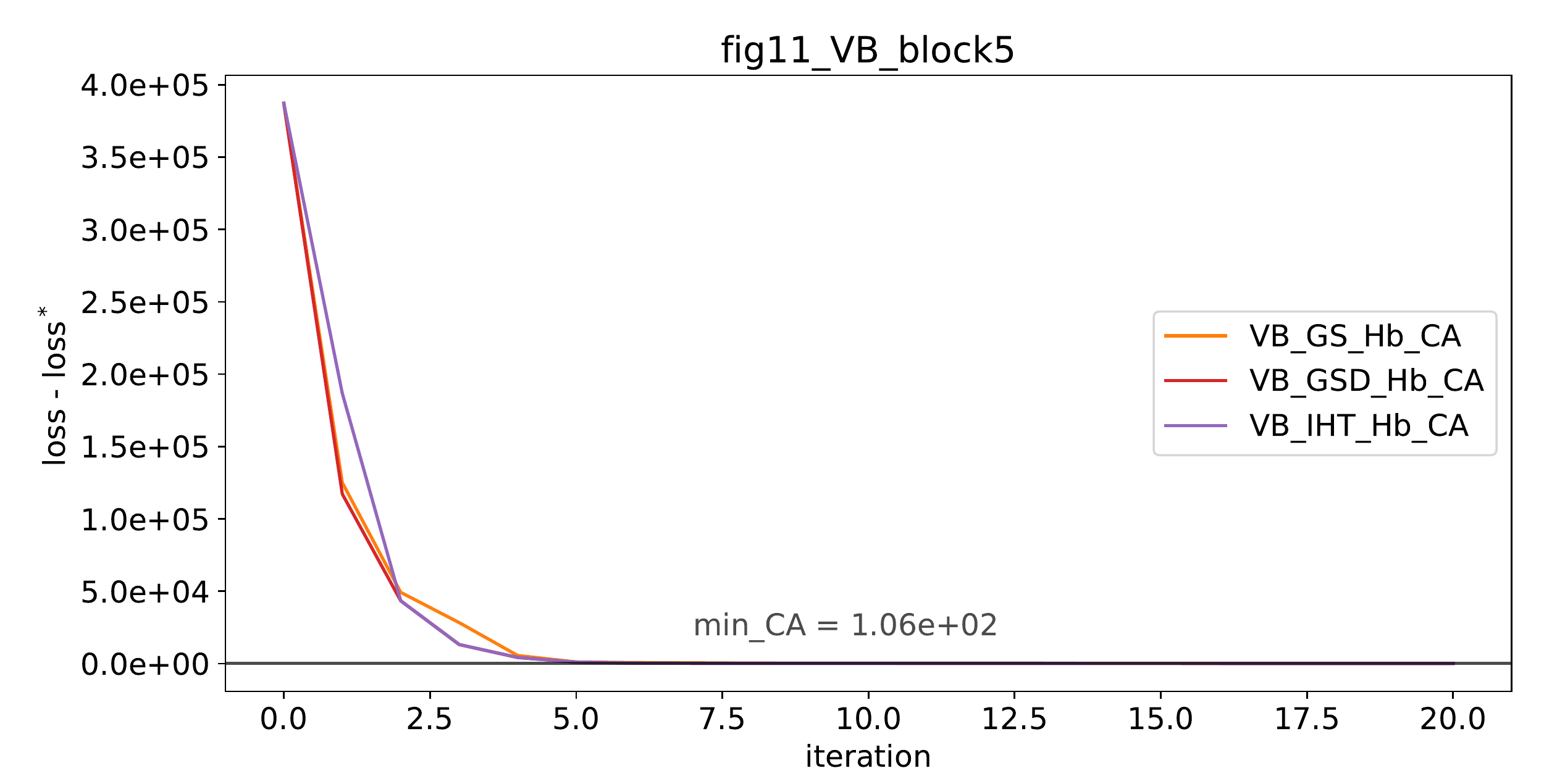}
\includegraphics[width=0.32\textwidth, clip=true, trim = 12mm 125mm 0 0]{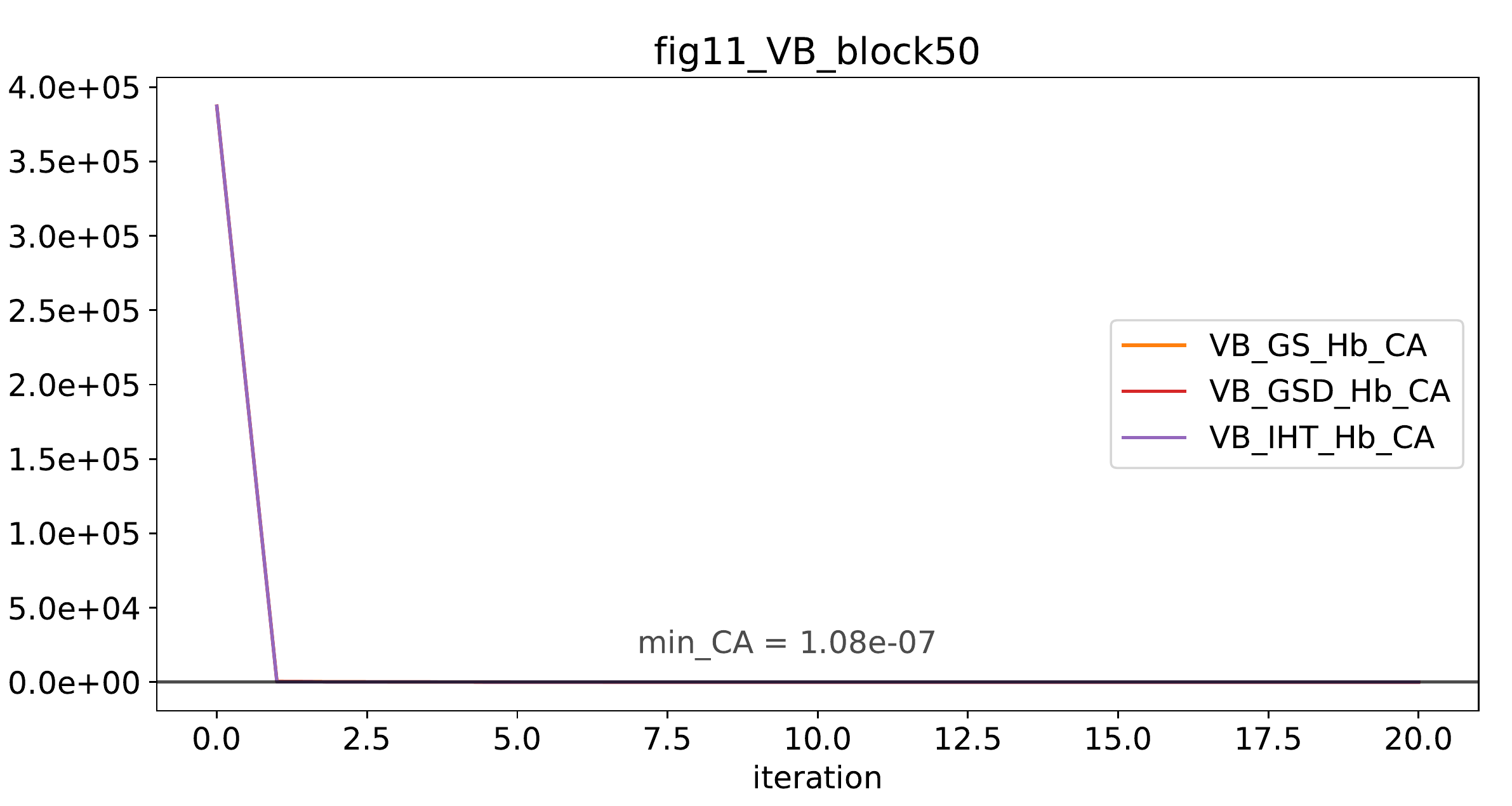}
\includegraphics[width=0.32\textwidth, clip=true, trim = 12mm 125mm 0 0]{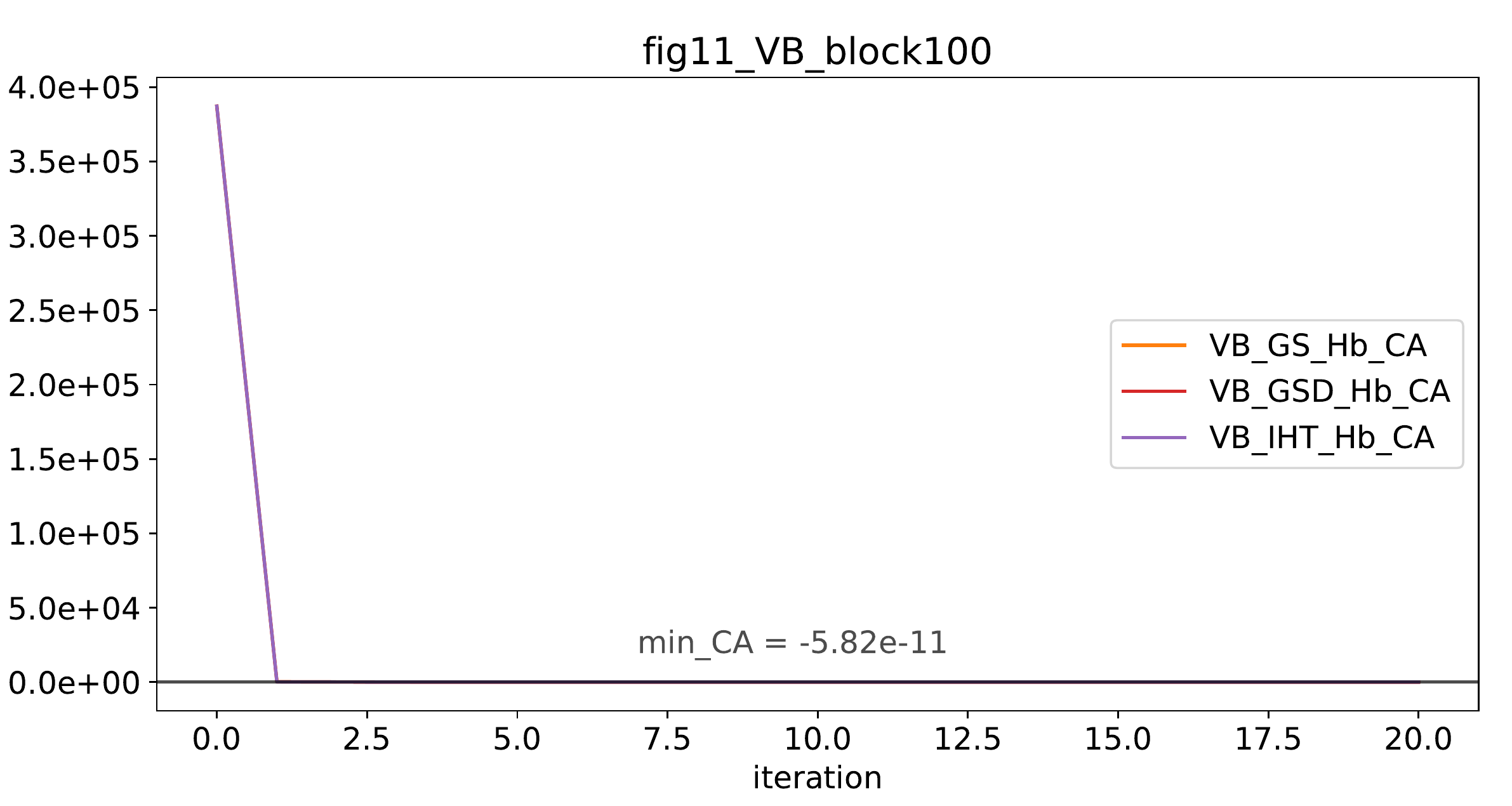}
\caption{CaBCD applied with different block sizes, rules used in \cite[Figure 5]{Nutini2017} (top-line) and 
\cite[Figure 11]{Nutini2017}.}
\label{fig:nutinifig5,11}
\end{figure}
\subsection{Discussion of results}
{
  The results show that the general conclusion of \cite{Nutini2017} also applies to CaBCD. 
Firstly, from Figure~\ref{fig:nutinifig4,8,9,10} the GS based rules should be preferred when possible. 
Secondly, from Figure~\ref{fig:nutinifig4,8,9,10} it can be observed that between the partition rules we should prefer VB or Sort. 
In our experiments, the differences between the partition rules VB and Sort are less evident. In particular, we can notice that the VB partition rule attains its minimum loss when the block size is $5$, which is congruent with our observation of Section~\ref{sec:exp_BCD} that the CaBCD makes more steps with the reduced measure when the block's size is low. 
Thirdly, from Figure~\ref{fig:nutinifig4,8,9,10} and~\ref{fig:nutinifig5,11} the differences between the selection rules vanish when the blocks' size increases. 
Lastly, the (quasi-)Newton updates $Hb$ Figure~\ref{fig:nutinifig5,11} reach a lower minimum faster, as one can expect. 
However, we recall that the Carath\'eodory reduced measure was built matching only the gradient and in the future, we want to refine this aspect applying the Carath\'eodory Sampling ``exactly'' also to the second derivative, i.e. (quasi-)Newton methods. 
}
\section{Summary}
We introduced a new SGD algorithm, CaGD and then combined it with BCD to make it scalable to high-dimensional spaces.
Similar to SGD variants we approximate the gradient in each descent step by a subset of the data.
In contrast to such SGD variants, the approximation is not done by randomly selecting a small subset of points and giving each point the same, uniform weight;
instead the points are carefully selected from the original dataset and weighing them differently.
This recombination step results in a small, weighted summary of the data is constructed and subsequently the gradient is only computed using this simpler summary until a control statistic tells us to recombine again.
To deal with high-dimensional optimization problems we then leveraged the strengths of this approach (low-variance gradient estimates) with BCD (low computational complexity).
Our experiments show that this can lead to remarkable improvements compared to competitive baselines such as ADAM and SAG. 
Many extensions are possible, e.g.~on the theoretical side, studying the behaviour under non-convex losses and on the applied side, combination with Quasi-Newton methods, or BCD rules that are specialized to CaBCD.
Independently of these, any improvement for recombination algorithms can lead to a further speed up of CaGD resp.~CaBCD.

\paragraph{Acknowledgements.}
The authors want to thank The Alan Turing Institute and the University of Oxford for the financial support given. FC is supported by The Alan Turing Institute, TU/C/000021, under the EPSRC Grant No. EP/N510129/1. HO is supported by the EPSRC grant ``Datasig'' [EP/S026347/1], The Alan Turing Institute, and the Oxford-Man Institute.

\bibliographystyle{plain}
\addcontentsline{toc}{section}{References}
\bibliography{Biblio}

\end{document}